\documentclass{article}

\PassOptionsToPackage{numbers,compress}{natbib}
\usepackage[final]{neurips_2026}

\usepackage[utf8]{inputenc} 
\usepackage[T1]{fontenc}    
\usepackage{hyperref}       
\usepackage{url}            
\usepackage{booktabs}       
\usepackage{amsfonts}       
\usepackage{nicefrac}       
\usepackage{microtype}      
\usepackage{xcolor}         
\usepackage{graphicx}
\usepackage{float}
\usepackage{wrapfig}
\usepackage{subcaption}
\usepackage{multirow}
\usepackage{amsmath}
\usepackage{amssymb}
\usepackage{mathtools}
\usepackage{amsthm}
\usepackage{pgfplots}
\pgfplotsset{compat=1.18}
\definecolor{npgblue}{RGB}{60,84,136}
\definecolor{npggreen}{RGB}{0,160,135}
\definecolor{npgred}{RGB}{230,75,53}

\makeatletter

\expandafter\let\csname algorithm*\endcsname\relax
\expandafter\let\csname endalgorithm*\endcsname\relax
\makeatother
\usepackage[ruled,vlined]{algorithm2e}

\usepackage[capitalize,noabbrev]{cleveref}

\theoremstyle{plain}
\newtheorem{theorem}{Theorem}[section]
\newtheorem{proposition}[theorem]{Proposition}
\newtheorem{lemma}[theorem]{Lemma}

\theoremstyle{definition}

\theoremstyle{remark}

\usepackage[textsize=tiny]{todonotes}

\title{HoReN: Normalized Hopfield Retrieval for Large-Scale Sequential Model Editing}

\author{%
  Yuan Fang\textsuperscript{1} \quad
  Yi Xie\textsuperscript{2} \quad
  Xuming Ran\textsuperscript{3}\thanks{Corresponding author: \href{mailto:ranxuming@gmail.com}{ranxuming@gmail.com}} \\[0.3em]
  \textsuperscript{1}IXL Learning, Inc \quad
  \textsuperscript{2}Technical University of Munich \quad
  \textsuperscript{3}National University of Singapore
}

\makeatletter\def\@trackname{}\makeatother
\begin{document}

\maketitle


\begin{abstract}
Large language models encode vast factual knowledge that can become outdated or incorrect after deployment, yet retraining is prohibitively costly. This motivates lifelong model editing, which updates targeted behavior while preserving the rest of the model. Existing editors, both parameter-modifying and parameter-preserving, degrade severely as edits accumulate and struggle to generalize across paraphrases. We propose \textbf{HoReN}, a codebook-based parameter-preserving editor that wraps a single MLP layer with a discrete key--value memory. HoReN treats each codebook entry as both a knowledge key and a Hopfield stored pattern, retrieves edits by angular similarity on the unit hypersphere, and refines queries through damped Hopfield dynamics so paraphrases converge to the correct memory basin while unrelated inputs remain stable. HoReN achieves strong editing performance with consistent gains across diverse benchmarks spanning standard ZsRE, structured WikiBigEdit, and unstructured UnKE evaluations. Moreover, HoReN scales to 50K sequential edits on ZsRE with stable overall performance above 0.93, while prior editors collapse or degrade severely before reaching 10K. Our code is available at \url{https://github.com/ha11ucin8/HoReN}.
\end{abstract}

\section{Introduction}
\label{sec:introduction}
Large language models (LLMs)~\citep{Dubey2024TheL3,Yang2024Qwen25TR,DeepSeekAI2025DeepSeekR1IR,Agarwal2025gptoss120bgptoss20bMC} are routinely deployed against a world whose facts drift, whose deployments surface hallucinations, and whose operators need to inject corrections that were absent at pre-training time. Re-training end-to-end for every update is prohibitively expensive and risks catastrophic forgetting~\citep{kirkpatrick2017overcoming,Ran2024BraininspiredCP,Xue2025DistillationGuidedST,Luo2026KeyValuePC}, motivating \emph{model editing}: targeted procedures that revise an LLM's behavior on a specified prompt while leaving its other behaviors intact, with edits arriving as a continual stream~\citep{Meng2022LocatingAE,hartvigsen2023aging,Wang2024WISERT,wang2025repair,Luo2026ReversibleLM}. Existing editors fall into two paradigms: \textbf{parameter-modifying} methods alter base weights, including the locate-and-edit line (ROME~\citep{Meng2022LocatingAE}, MEMIT~\citep{Meng2022MassEditingMI}, AlphaEdit~\citep{Fang2024AlphaEditNC}) and one-shot direct shifts (UltraEdit~\citep{Gu2025UltraEditTS}); \textbf{parameter-preserving} methods leave the base weights frozen and route edits through external memory at a chosen layer, as in GRACE's discrete codebook~\citep{hartvigsen2023aging} and WISE's gated side-memory~\citep{Wang2024WISERT}.

\begin{figure}[t]
    \centering
    \includegraphics[width=\textwidth]{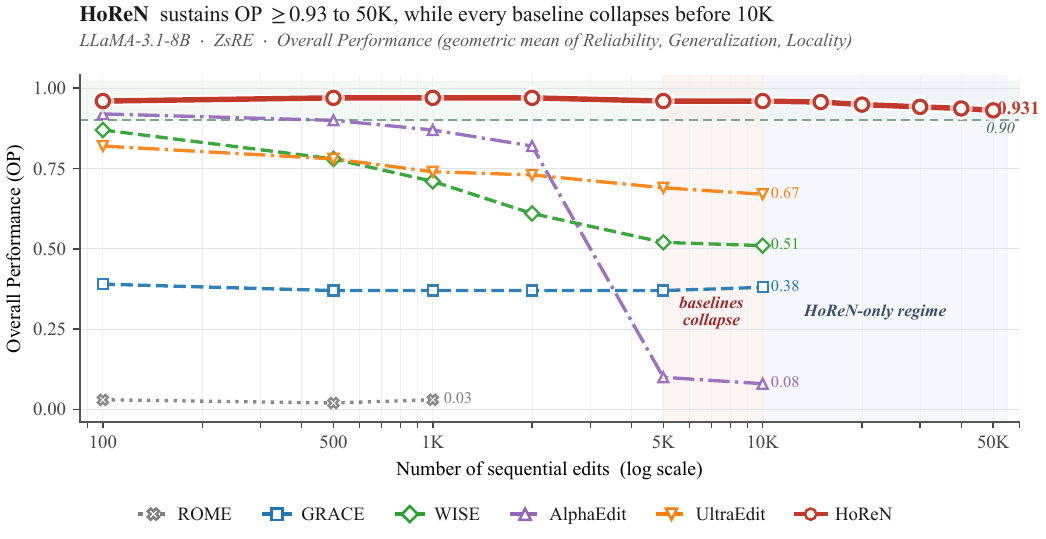}
    \caption{Overall Performance (OP, geometric mean of Reliability, Generalization, and Locality) scaling from $100$ to $50\text{K}$ sequential edits on ZsRE (LLaMA-3.1-8B). HoReN sustains OP $\geq 0.93$ through 50K edits, while every baseline collapses or plateaus before 10K: ROME fails immediately ($0.03$), GRACE flatlines at $0.37$--$0.38$, WISE drifts down to $0.51$, AlphaEdit cliffs from $0.82$ at 2K to $0.08$ at 10K, and UltraEdit drifts down to $0.67$. Beyond 10K only HoReN remains within the $0.90$ stability band. Per-metric breakdowns and the per-checkpoint table are in Appendix~\ref{app:llama-scaling}.}
    \label{fig:50k-scaling}
\end{figure}

Parameter-modifying editors have shown notable success on small batches of edits, but directly perturbing weights trained on trillions of tokens disrupts preserved knowledge as edits accumulate, and constraint-based remedies such as AlphaEdit's null-space projection~\citep{Fang2024AlphaEditNC} only partially mitigate this drift in a sequential regime (Figure~\ref{fig:50k-scaling}). Parameter-preserving editors take a different route. GRACE~\citep{hartvigsen2023aging} writes new mappings into a chosen layer's latent space as a discrete key--value codebook of adaptors and routes by a hard-radius nearest-neighbor lookup that defers to a stored value whenever an incoming activation falls within the radius of a stored key. WISE~\citep{Wang2024WISERT} maintains a dual parametric memory, a frozen main memory of pretrained weights and a trainable side memory holding edits with a router selecting between them, and uses a knowledge-sharding mechanism that places distinct edit subsets in disjoint subspaces before merging them into the side memory. Both designs protect existing knowledge by construction, yet their generalization and overall performance degrade as the edit stream scales (Figure~\ref{fig:50k-scaling}). The failure is fundamentally a routing problem: the lookup or router cannot reliably map a paraphrased query to the correct stored edit.

To address these routing failures in parameter-preserving editors, we propose \textbf{HoReN} (\textbf{Ho}pfield \textbf{Re}trieval with \textbf{N}ormalized representations), which combines three components. (i) The \emph{codebook architecture} of GRACE~\citep{hartvigsen2023aging}: a frozen base model with one MLP layer wrapped by a discrete key--value memory, so each edit is stored once and never re-touched. (ii) \emph{Normalization}: keys and queries are projected onto the unit hypersphere so retrieval depends on angular similarity rather than activation magnitude. (iii) \emph{Hopfield dynamics}~\citep{Ramsauer2020HopfieldNI,Krotov2023ANF}: the codebook is read as a field of energy basins, and one damped attractor step before matching lets a paraphrase relax into the basin of the correct memory while leaving unrelated queries essentially undisturbed. The matched edit is realized by a lightweight value adaptor at the same layer $l$, so the entire intervention is concentrated on the routing decision.

The two novel components play complementary roles: normalization protects locality and enables generalization by keeping unrelated queries separated on the hypersphere, while one damped Hopfield step closes the generalization gap by pulling rephrased queries into the basin of the correct edit. The result is a sequential editor that scales without catastrophic degradation: on a controlled ZsRE stress test on LLaMA-3.1-8B, HoReN keeps reliability, generalization, and locality all above $89\%$ out to 50K accumulated edits (Figure~\ref{fig:50k-scaling}), where every baseline has cliffed (AlphaEdit, ROME), drifted (UltraEdit, WISE), or flat-lined on generalization (GRACE) well before 10K. The same advantage carries over to structured WikiBigEdit~\citep{Thede2025WikiBigEditUT} and unstructured UnKE~\citep{Deng2024UnKEUK}. Related works are shown in Appendix ~\ref{app:related-work}.

Our main contributions are:
\begin{itemize}
\item We characterize a structural divide in lifelong editing: parameter-modifying editors degrade preserved knowledge even with constraint measures, while parameter-preserving editors lose generalization because magnitude-sensitive nearest-neighbor routing cannot bridge the representation gap between an edit prompt and its rephrasings.
\item We propose HoReN, a parameter-preserving editor that unifies the ROME key--value memory view with the GRACE codebook architecture under a normalized Hopfield retrieval rule, and provide theoretical analysis showing that iterating Hopfield dynamics to convergence attracts all queries, including unrelated ones, toward stored codes, formally motivating HoReN's single-step deployment.
\item We demonstrate large-scale stability of HoReN on a 50K-edit ZsRE stress test on LLaMA-3.1-8B, where prior methods cliff, drift, or flat-line well before 10K edits, and confirm consistent gains on structured WikiBigEdit and unstructured UnKE across seven LLMs spanning four model families.
\end{itemize}

\begin{figure*}[t]
    \centering
    \includegraphics[width=\textwidth]{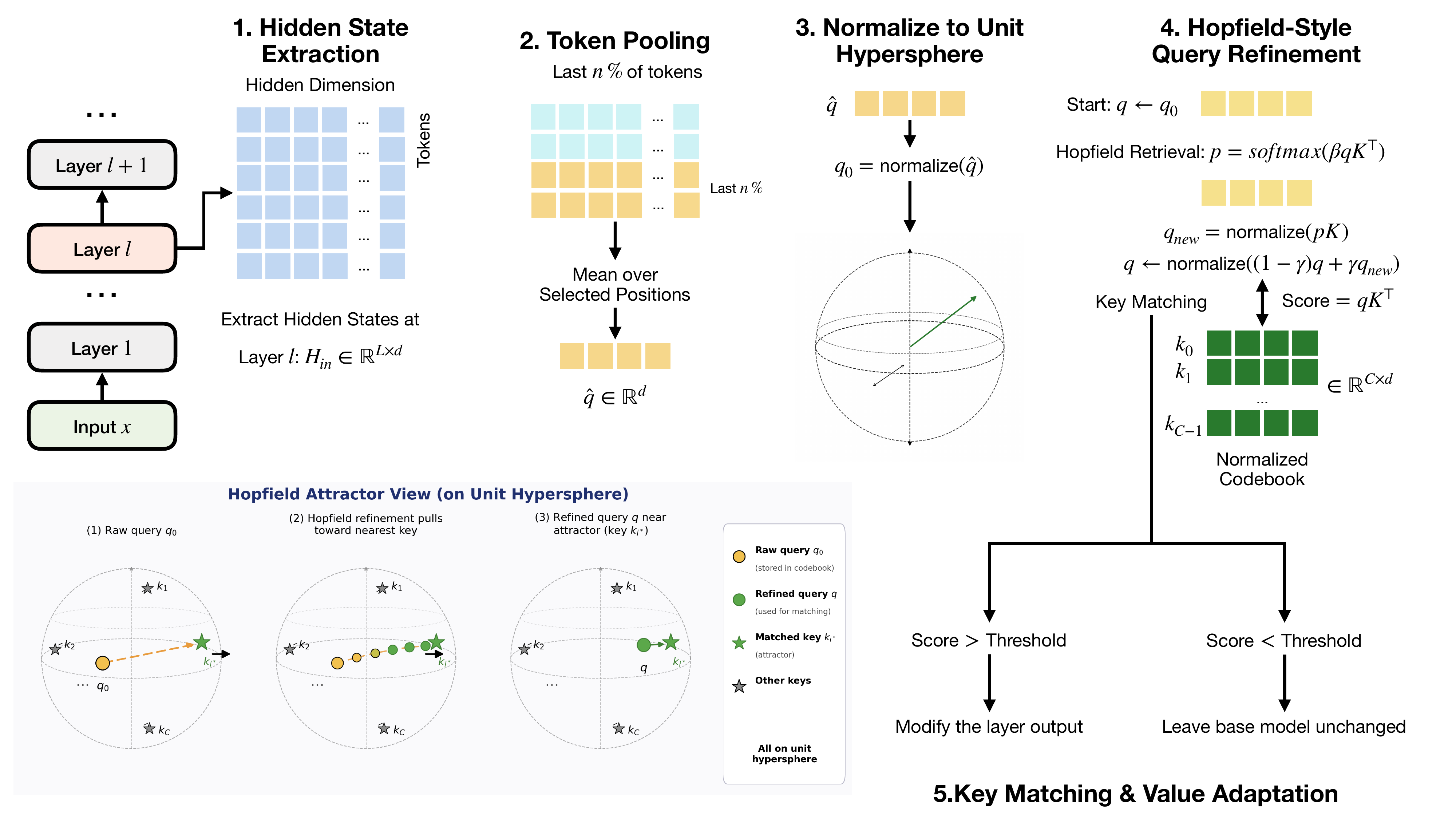}
    \caption{Overall architecture of HoReN combining normalized representations with Hopfield-style retrieval. The pipeline: \textbf{(1)} extract hidden states from layer $l$; \textbf{(2)} pool token representations to construct the query; \textbf{(3)} normalize query to the unit hypersphere; \textbf{(4)} apply one-step Hopfield-style refinement over the normalized codebook; \textbf{(5)} perform key matching and value adaptation at layer $l$. This design achieves complementary benefits: normalization preserves locality through angular similarity matching, while Hopfield-style refinement improves paraphrase routing by moving the query toward edit-key attractors.}
    \label{fig:overview}
\end{figure*}

\section{Method}
\label{sec:method}

HoReN inherits the architectural skeleton of GRACE~\citep{hartvigsen2023aging}: the base model $f_{\theta_0}$ is frozen and a single MLP layer $l$ is wrapped with a discrete codebook $\mathcal{C}=\{(k_i,v_i,y_i)\}_{i=1}^{C}$ (Figure~\ref{fig:overview}). At inference, an incoming query is routed through the codebook; a successful match modifies the layer's output by adding the matched value to the output, while an unmatched query passes through the base model unchanged. Editing reduces to two questions: \emph{what is stored} and \emph{how is a query routed to the right entry}. HoReN's answers to both flow from treating each stored key simultaneously as a ROME-style knowledge-memory index~\citep{Meng2022LocatingAE} and as an attractor basin in a modern Hopfield network~\citep{Ramsauer2020HopfieldNI}: keys and queries are projected onto the unit hypersphere so that retrieval is purely angular (Section~\ref{sec:codebook}), and routing proceeds by a single controlled relaxation under the codebook's attractor field before any matching decision (Section~\ref{sec:hopfield-formulation}).

\subsection{Problem Formulation}
\label{sec:problem-formulation}
\textbf{Setup.} Following the lifelong editing setup of~\citep{Wang2024WISERT}, let $f_{\theta}:\mathcal{X}\rightarrow\mathcal{Y}$ be an LLM with pre-edit parameters $\theta_0$. Each edit index $t$ carries an original prompt $x_t$, a paraphrase $\tilde{x}_t$, a shared target $y_t^*$, and an unrelated locality query $x_t^{\mathrm{loc}}$. At step $t$ the editor produces
\begin{equation}
f_{\theta_t}=\mathrm{ModelEditing}(f_{\theta_{t-1}},x_t,y_t^*),
\quad\text{so that}\quad
f_{\theta_t}(x)=
\begin{cases}
y_t^*, & x\in\{x_t,\tilde{x}_t\},\\
f_{\theta_0}(x), & x=x_t^{\mathrm{loc}}.
\end{cases}
\end{equation}
After $N$ sequential edits, the final model must satisfy this condition for \emph{all} $t\in\{1,\ldots,N\}$, measured by reliability, generalization, and locality (definitions and dataset-specific instantiations in Section~\ref{sec:expii-setup}). The difficulty of the problem grows with $N$: each new edit must coexist with every prior one, so the editor needs a memory that scales gracefully and a routing rule that remains discriminative even as nearby memories accumulate.

\subsection{The Codebook: Normalized Keys on the Unit Hypersphere}
\label{sec:codebook}
Each codebook entry $(k_i,v_i,y_i)$ at editing layer $l$ binds a key $k_i\in\mathbb{R}^d$ to a value vector $v_i\in\mathbb{R}^d$ and the target label $y_i$ that originally created the entry. Keys are constructed from the layer-$l$ hidden states of the prompt that triggered the edit, and queries are constructed in the same way at routing time, so that storage and retrieval live in a common representation space.

Concretely, given an input $x$ we extract $H_{\text{in}}\in\mathbb{R}^{L\times d}$ from the frozen base at layer $l$ and pool the last $n\%$ of token positions into a raw query $\hat q=\mathrm{select}(H_{\text{in}})\in\mathbb{R}^{d}$. Pooling a suffix of tokens, rather than a single position, summarises the semantic content of the prompt while remaining robust to surface variation such as padding and minor rephrasing. We then project to the unit hypersphere,
\[
q_0 = \mathrm{normalize}(\hat q)\in\mathbb{R}^{1\times d},\qquad K\in\mathbb{R}^{C\times d},\;\;\|k_i\|_2=1,
\]
where $C=|\mathcal{C}|\le N$ is the current codebook size. The motivation is the geometric fact that, for post-activation MLP tensors, the direction of the vector is what binds to a particular concept while the magnitude largely reflects prompt-dependent activation strength; normalising removes this nuisance variation and exposes the angular structure that the keys are actually selecting on.

Once a query has been routed (Section~\ref{sec:hopfield-formulation}), the matching decision itself is a single inner product: the editor scores keys as $\mathbf{a}=qK^\top$, takes $i^\star=\arg\max_i \mathbf{a}_i$, and accepts the match when $\mathbf{a}_{i^\star}>c$. Because keys and queries lie on the unit sphere, $c$ is a cosine threshold and admits a precise geometric reading: it is the angular boundary of the attractor basin around $k_{i^\star}$. This replaces GRACE's deferral radius, a Euclidean tolerance whose meaning shifts with activation magnitude, with a single, scale-free notion of what it means for a query to fall ``inside'' a stored memory.

\subsection{Querying via Hopfield Attractor Dynamics}
\label{sec:hopfield-formulation}
Treating each stored key as an attractor suggests that routing should not be a one-shot nearest-neighbour lookup against the raw query $q_0$. A paraphrase $\tilde x_t$ produces an activation that is angularly close to, but not identical to, the activation of the original prompt $x_t$; under a hard argmax this small angular gap can be enough to send the query to the wrong key once the codebook is dense. The right object to compare against the codebook is therefore not $q_0$ itself but the result of letting $q_0$ relax briefly under the codebook's attractor field, so that paraphrases drift toward the basin of the correct memory while unrelated queries are left essentially untouched.

We make this precise using the modern Hopfield framework~\citep{Ramsauer2020HopfieldNI}. With energy
\[
E(q,K) \;=\; \tfrac12\|q\|_2^2 - \mathrm{lse}_\beta\!\big(qK^\top\big),\qquad \mathrm{lse}_\beta(\mathbf{z})=\tfrac{1}{\beta}\log\textstyle\sum_i e^{\beta z_i},
\]
the standard update $T(q)=\mathrm{softmax}(\beta\, qK^\top)\,K$ descends $E$ and converges to fixed points that lie in the convex hull of the stored patterns, with $\beta>0$ controlling retrieval sharpness. HoReN deploys a \emph{normalized, damped} variant of this dynamic. Starting from $q\leftarrow q_0$, for $m=1,\ldots,M$ we compute
\begin{equation}
\begin{aligned}
p &= \mathrm{softmax}(\beta\, qK^\top) \in \mathbb{R}^{1\times C}, \\
q_{\text{new}} &= \mathrm{normalize}(pK) \in \mathbb{R}^{1\times d},
\end{aligned}
\end{equation}
and update
\begin{equation}
q \;\leftarrow\; \mathrm{normalize}\!\big((1-\gamma)\, q + \gamma\, q_{\text{new}}\big),
\end{equation}
terminating early once $\|q_{\text{new}}-q\|_2\le\epsilon$. Two deviations from the standard update matter. The intermediate re-normalisation $q_{\text{new}}=\mathrm{normalize}(pK)$ keeps the iterate on the same hypersphere as the keys, so the dynamics never leave the geometry on which the codebook is defined. The damping factor $\gamma\in(0,1]$ moves the query only fractionally toward $q_{\text{new}}$ at each step, so unrelated queries, whose attention weights $p$ are diffuse, are perturbed only slightly, while paraphrases of an existing edit, whose $p$ concentrates on a single key, are pulled into the corresponding basin.

The single damped step is not an empirical convenience but a principled stopping criterion. Iterating $T$ to convergence would attract \emph{any} query, including unrelated or genuinely novel ones, into the convex hull of stored patterns, which for large $\beta$ collapses near the closest existing key and therefore destroys locality. We characterise this behaviour with two complementary results, stated here and proved in Appendix~\ref{sec:proofs}.

\begin{theorem}[Asymptotic convergence of standard Hopfield retrieval]
\label{thm:hopfield_asymptotic_convergence}
Let $K \in \mathbb{R}^{C \times d}$ be a normalized codebook ($\|k_i\|_2 = 1$), let $q^{(0)} \in \mathbb{R}^{1 \times d}$ satisfy $\|q^{(0)}\|_2 = 1$, and let $\beta>0$. Define the standard Hopfield update map $T(q):=\mathrm{softmax}(\beta\, q K^\top)\,K$ and generate the infinite sequence $q^{(s+1)}=T(q^{(s)})$. With $F(q):=\tfrac{1}{\beta}\log\sum_{i=1}^{C}\exp(\beta\, q k_i^\top)$ and energy $E(q,K):=\tfrac12\|q\|_2^2-F(q)$, the energy is monotonically non-increasing,
\[
E(q^{(s+1)},K)-E(q^{(s)},K)\;\le\;-\tfrac12\,\|q^{(s+1)}-q^{(s)}\|_2^2,
\]
and the sequence converges to a fixed point: $q^{(s)}\to q^\ast$ with $q^\ast=T(q^\ast)$.
\end{theorem}

\begin{proposition}[Finite-step descent and residual bound]
\label{prop:hopfield_finite_step}
Under the assumptions of Theorem~\ref{thm:hopfield_asymptotic_convergence}, for any finite $M\ge 1$ the truncated iterates $q^{(s+1)}=T(q^{(s)})$, $s=0,\ldots,M-1$, satisfy the cumulative descent inequality
\[
\sum_{s=0}^{M-1}\|q^{(s+1)}-q^{(s)}\|_2^2 \;\le\; 2\bigl(E(q^{(0)},K)-E(q^{(M)},K)\bigr),
\]
and consequently the iterate with smallest fixed-point residual obeys the explicit bound
\[
\min_{0\le s<M}\,\|T(q^{(s)})-q^{(s)}\|_2 \;\le\; \frac{2}{\sqrt{M}}.
\]
The bound applies to the \emph{best} of the first $M$ iterates; the last iterate $q^{(M)}$ is only guaranteed to satisfy $E(q^{(M)},K)\le E(q^{(0)},K)$ and may not itself be an approximate fixed point.
\end{proposition}

Theorem~\ref{thm:hopfield_asymptotic_convergence} formalises the over-attraction risk: in the limit, every query is dragged into the convex hull of stored codes regardless of whether it actually corresponds to an edit. Proposition~\ref{prop:hopfield_finite_step} shows that this contraction is already operative at finite step counts in an averaged sense, so even partial iteration progressively erodes locality. The normalized, damped single-step variant deployed by HoReN preserves the generalisation benefit of attractor dynamics, and paraphrases are pulled toward the right basin, while controlling exactly this over-attraction; $M$ then exposes a principled generalisation/locality trade-off whose empirical behaviour we study in Section~\ref{sec:expii-analysis-ablation} (Figure~\ref{fig:hopfield-steps}, Table~\ref{tab:normalized-codebook}).

\subsection{Updating the Codebook}
\label{sec:codebook-update}
The codebook grows online as edits arrive. After routing, three cases can arise. (i) No-match: the best score $\mathbf{a}_{i^\star}\le c$ falls outside every existing basin, so the edit is genuinely new and we insert a fresh entry $(q_0, v_{\text{new}}, y_t^*)$, where $v_{\text{new}}$ is trained by cross-entropy on the target $y_t^*$ with the base model frozen. (ii) Match with consistent label ($y_{i^\star}=y_t^*$): the edit is a paraphrase or a reassertion, and we fine-tune the existing value adaptor $v_{i^\star}$ on the new prompt while keeping its key $k_{i^\star}$ fixed. (iii) Match with conflicting label ($y_{i^\star}\neq y_t^*$): the basin is contested and we insert a new entry to preserve the new fact, leaving the prior memory in place. At inference, a successful match adds the matched value to the layer output at the index of the last prompt token to be passed to the next layer; on a no-match the base model is left unchanged. An alternative LoRA-based adaptor is described in Appendix~\ref{app:lora-vs-value}, and the full procedure is given in Algorithm~\ref{alg:hopfield_edit_simple} (Appendix~\ref{app:algorithm}).

\paragraph{Routing mechanism.} The routing mechanism has a key structural asymmetry: stored keys are always the \emph{pre-refinement} query $q_0$, while the routing comparison uses the refined $q$ output by the Hopfield step. This asymmetry is not a tuning choice but follows directly from the attractor geometry. If newly stored keys were the output of the Hopfield refinement, each insertion would pull subsequent insertions toward existing attractors, basins would co-move with the codebook, and the angular separation that retrieval depends on would progressively collapse. Pinning storage to $q_0$ anchors the codebook to the model's native representational geometry, so the attractor field is fixed by what the base model actually produces; the refinement step then moves paraphrases toward the correct basin at routing time without contaminating what is stored.

\section{Experiments}
\label{sec:experiments}

\subsection{Experimental Setup}
\label{sec:expii-setup}

We briefly outline the base LLMs, baselines, datasets, and evaluation metrics. Full setup descriptions and implementation details are in Appendix~\ref{app:experimental-details}.

\paragraph{Base LLMs \& Baseline Methods.} Our experiments are conducted on seven LLMs spanning 1.5B to 32B parameters across five families: DeepSeek-R1-Distill-Qwen-1.5B~\citep{DeepSeekAI2025DeepSeekR1IR}, LLaMA-3-8B-Instruct and LLaMA-3.1-8B-Instruct~\citep{Dubey2024TheL3}, Qwen2.5-7B-Instruct~\citep{Yang2024Qwen25TR}, DeepSeek-R1-Distill-Llama-8B (hereafter DeepSeek-R1-8B)~\citep{DeepSeekAI2025DeepSeekR1IR}, GPT-OSS-20B~\citep{Agarwal2025gptoss120bgptoss20bMC}, and Qwen-SEA-LION-v4-32B-IT~\citep{Ng2025SEALIONSA}. We compare HoReN against representative methods from two paradigms: (1) \emph{parameter-modifying} methods, including ROME~\citep{Meng2022LocatingAE}, AlphaEdit~\citep{Fang2024AlphaEditNC}, and UltraEdit~\citep{Gu2025UltraEditTS}; and (2) \emph{parameter-preserving} methods, including GRACE~\citep{hartvigsen2023aging}, WISE~\citep{Wang2024WISERT}, and REPAIR~\citep{wang2025repair}. For unstructured editing we additionally compare against the domain-specific UnKE baseline~\citep{Deng2024UnKEUK}.

\paragraph{Datasets \& Evaluation Metrics.} We evaluate HoReN on three benchmarks that progressively stress its core separation guarantee: \textbf{ZsRE}~\citep{levy2017zero}, the standard lifelong editing benchmark with semantically well-separated edit and locality queries; \textbf{WikiBigEdit}~\citep{Thede2025WikiBigEditUT}, a harder regime in which locality queries share subjects and relations with edited facts; and \textbf{UnKE}~\citep{Deng2024UnKEUK}, an unstructured benchmark with free-form edit targets. Following prior work~\citep{hartvigsen2023aging,Wang2024WISERT,wang2025repair}, we report \textbf{Reliability}, \textbf{Generalization}, and \textbf{Locality}, together with their geometric mean (Overall Performance, OP). For UnKE we follow its original protocol~\citep{Deng2024UnKEUK} and report BLEU, ROUGE-1/2/L and BERTScore. Full dataset and metric definitions are in Appendices~\ref{app:datasets} and~\ref{app:metrics}.

\subsection{ZsRE Baseline Failure Modes and Large-Scale Stability}
\label{sec:expii-main-results}

At $N{=}1000$, every baseline exhibits at least one critical failure mode on ZsRE; Tables~\ref{tab:zsre-comparison}--\ref{tab:scaling-10k} and Figure~\ref{fig:50k-scaling} show that these failures either accelerate (cliff) or persist (slope) as edits scale to 50K, while HoReN maintains all three metrics throughout.

\subsubsection{Baseline Comparison on ZsRE}
\label{sec:expii-zsre}

\begin{table*}[t]\scriptsize
\centering
\caption{Performance comparison on ZsRE benchmark at different editing scales ($N \in \{100, 500, 1000\}$) on four models.}
\label{tab:zsre-comparison}
\resizebox{\textwidth}{!}{%
\begin{tabular}{ll|cccc|cccc|cccc}
\toprule
\textbf{Model} & \textbf{Method}
& \multicolumn{4}{c|}{\textbf{$N = 100$}}
& \multicolumn{4}{c|}{\textbf{$N = 500$}}
& \multicolumn{4}{c}{\textbf{$N = 1000$}} \\
\cmidrule(lr){3-6} \cmidrule(lr){7-10} \cmidrule(lr){11-14}
&
& \textbf{Rel.} & \textbf{Gen.} & \textbf{Loc.} & \textbf{OP}
& \textbf{Rel.} & \textbf{Gen.} & \textbf{Loc.} & \textbf{OP}
& \textbf{Rel.} & \textbf{Gen.} & \textbf{Loc.} & \textbf{OP} \\
\midrule
\multirow{6}{*}{\begin{tabular}[c]{@{}l@{}}LLaMA-3\\8B\end{tabular}}
& ROME & 0.17	& 0.16 & 0.01 & 0.06 & 0.02	& 0.03 & 0.00 & 0.00 & 0.02 & 0.02 & 0.00 & 0.00 \\
& GRACE & 0.99 & 0.01 & 1.00 & 0.21 & 0.99 & 0.02 & 1.00 & 0.27 & 0.99 & 0.02 & 1.00 & 0.27 \\
& WISE & 0.54 & 0.50 & 1.00 & 0.65 & 0.56 & 0.54 & 1.00 & 0.67 & 0.55 & 0.54 & 0.99 & 0.66 \\
& AlphaEdit & 0.98 & 0.92 & 0.89 & 0.93 & 0.99 & 0.94 & 0.83 & 0.92 & 0.98 & 0.92 & 0.78 & 0.89 \\
& UltraEdit & 0.91 & 0.88 & 0.69 & 0.82 & 0.90 & 0.88 & 0.56 & 0.76 & 0.87 & 0.85 & 0.54 & 0.74 \\
& \textit{HoReN} & \textbf{1.00} & \textbf{0.81} & \textbf{1.00} & \textbf{0.93} & \textbf{1.00} & \textbf{0.84} & \textbf{1.00} & \textbf{0.94} & \textbf{0.99} & \textbf{0.84} & \textbf{1.00} & \textbf{0.94} \\
\midrule
\multirow{6}{*}{\begin{tabular}[c]{@{}l@{}}LLaMA-3.1\\8B\end{tabular}}
& ROME & 0.04 & 0.05 & 0.01 & 0.03 & 0.02 & 0.02 & 0.02 & 0.02 & 0.03 & 0.02 & 0.03 & 0.03 \\
& GRACE & 1.00 & 0.06 & 1.00 & 0.39 & 1.00 & 0.05 & 1.00 & 0.37 & 0.99 & 0.05 & 1.00 & 0.37 \\
& WISE & 0.83 & 0.79 & 1.00 & 0.87 & 0.71 & 0.66 & 1.00 & 0.78 & 0.62 & 0.59 & 0.99 & 0.71 \\
& AlphaEdit & 0.98 & 0.90 & 0.88 & 0.92 & 0.98 & 0.93 & 0.80 & 0.90 & 0.97 & 0.92 & 0.74 & 0.87	 \\
& UltraEdit & 0.91 & 0.89 & 0.69 & 0.82 & 0.89 & 0.87 & 0.61 & 0.78 & 0.88 & 0.86 & 0.54 & 0.74 \\
& \textit{HoReN} & \textbf{0.99} & \textbf{0.88} & \textbf{1.00} & \textbf{0.96} & \textbf{1.00} & \textbf{0.92} & \textbf{1.00} & \textbf{0.97} & \textbf{0.99} & \textbf{0.93} & \textbf{0.99} & \textbf{0.97} \\
\midrule
\multirow{6}{*}{\begin{tabular}[c]{@{}l@{}}Qwen-2.5\\7B\end{tabular}}
& ROME & 0.51 & 0.51 & 0.32 & 0.44 & 0.30 & 0.28 & 0.06 & 0.17 & 0.19 & 0.19 & 0.02 & 0.09 \\
& GRACE & 1.00 & 0.02 & 1.00 & 0.27 & 1.00 & 0.02 & 1.00 & 0.27 & 1.00 & 0.02 & 1.00 & 0.27 \\
& WISE & 0.90 & 0.87 & 1.00 & 0.92 & 0.87 & 0.80 & 0.80 & 0.82 & 0.63 & 0.58 & 0.60 & 0.60 \\
& AlphaEdit & 0.86 & 0.80 & 0.98 & 0.88 & 0.85 & 0.79 & 0.91 & 0.85 & 0.99 & 0.88 & 0.88 & 0.92 \\
& UltraEdit & 0.92 & 0.80 & 0.99 & 0.90 & 0.87 & 0.81 & 0.96 & 0.88 & 0.80 & 0.75 & 0.89 & 0.81 \\
& \textit{HoReN} & \textbf{0.99} & \textbf{0.93} & \textbf{1.00} & \textbf{0.97} & \textbf{1.00} & \textbf{0.93} & \textbf{0.99} & \textbf{0.97} & \textbf{1.00} & \textbf{0.93} & \textbf{0.98} & \textbf{0.97} \\
\midrule
\multirow{6}{*}{\begin{tabular}[c]{@{}l@{}}SEA-LION\\32B\end{tabular}}
& ROME      & 0.02 & 0.02 & 0.02 & 0.02	 & 0.00 & 0.00 & 0.00 & 0.00 & 0.00 & 0.00 & 0.00 & 0.00 \\
& GRACE     & 0.46 & 0.02 & 1.00 & 0.21 & 0.41 & 0.02 & 1.00 & 0.20 & 0.43 & 0.02 & 1.00 & 0.21 \\
& WISE      & 0.94 & 0.72 & 1.00 & 0.88 & 0.80 & 0.68 & 1.00 & 0.82 & 0.70 & 0.63 & 1.00 & 0.76 \\
& AlphaEdit & 0.96 & 0.76 & 0.99 & 0.90 & 0.97 & 0.77 & 0.96 & 0.89 & 0.97 & 0.77 & 0.95 & 0.89	 \\
& UltraEdit & 0.76 & 0.70 & 0.98 & 0.80 & 0.75 & 0.72 & 0.95 & 0.80 & 0.72 & 0.69 & 0.93 & 0.77 \\
& \textit{HoReN} & \textbf{1.00} & \textbf{0.91} & \textbf{1.00} & \textbf{0.97} & \textbf{1.00} & \textbf{0.90} & \textbf{0.99} & \textbf{0.96} & \textbf{1.00} & \textbf{0.90} & \textbf{0.98} & \textbf{0.96} \\
\bottomrule
\end{tabular}%
}
\end{table*}

Each baseline in Table~\ref{tab:zsre-comparison} fails in one or more of reliability, generalization, and locality as edits accumulate.

\emph{Parameter-modifying editors.} ROME collapses on all three metrics almost immediately because it overfits the rank-one update to one prompt rather than to the underlying fact, and stacking such updates disturbs not only unrelated knowledge but also the edited entries (Rel.\ $0.03$, Gen.\ $0.02$, Loc.\ $0.03$ at $N{=}1000$ on LLaMA-3.1-8B). AlphaEdit and UltraEdit retain high reliability and generalization but their locality erodes as the edit stream grows: AlphaEdit's null-space projector is estimated from a finite sample of pre-training-style activations, so it protects directions in its proxy preservation sample but cannot include directions added by later edits, and locality therefore degrades fast ($0.88 \to 0.74$ from $N{=}100$ to $1000$ on LLaMA-3.1-8B); UltraEdit's running normalization stabilizes the update geometry but still disrupts preserved knowledge, resulting in the same gradual erosion in locality ($0.69 \to 0.54$ on LLaMA-3.1-8B).

\emph{Parameter-preserving editors.} GRACE inverts the failure: reliability and locality stay $\geq 0.99$ on every row, but generalization is stuck at $0.01$--$0.06$, because unnormalized nearest-neighbor lookup conflates the prompt-specific gain (magnitude) with the semantic identity (direction). WISE's gated side-memory drifts as the codebook accumulates, dragging both reliability and generalization down.

By capturing edit-semantic identity through normalization and refining queries through Hopfield dynamics, HoReN is the only method that maintains high reliability, generalization, and locality jointly (Rel.\ $0.99$, Gen.\ $0.93$, Loc.\ $0.99$ at $N{=}1000$ on LLaMA-3.1-8B).

\subsubsection{Scaling to 10K and 50K Edits}
\label{sec:expii-scaling}

\begin{table*}[t]\scriptsize
\centering
\caption{Scaling to 10K edits on ZsRE across four models (Qwen2.5-7B, DeepSeek-R1-8B, GPT-OSS-20B, and Qwen-SEA-LION-v4-32B). Full LLaMA-3.1-8B results at large scale, including AlphaEdit and UltraEdit, are in Appendix~\ref{app:llama-scaling}.}
\label{tab:scaling-10k}
\begin{tabular}{ll|cccc|cccc|cccc}
\toprule
\textbf{Model} & \textbf{Method}
& \multicolumn{4}{c|}{\textbf{$N = 2000$}}
& \multicolumn{4}{c|}{\textbf{$N = 5000$}}
& \multicolumn{4}{c}{\textbf{$N = 10000$}} \\
\cmidrule(lr){3-6} \cmidrule(lr){7-10} \cmidrule(lr){11-14}
&
& \textbf{Rel.} & \textbf{Gen.} & \textbf{Loc.} & \textbf{OP}
& \textbf{Rel.} & \textbf{Gen.} & \textbf{Loc.} & \textbf{OP}
& \textbf{Rel.} & \textbf{Gen.} & \textbf{Loc.} & \textbf{OP} \\
\midrule
\multirow{3}{*}{\begin{tabular}[c]{@{}l@{}}Qwen2.5\\7B\end{tabular}}
& GRACE      & 1.00 & 0.02 & 1.00 & 0.27 & 1.00 & 0.03 & 1.00 & 0.31 & 1.00 & 0.03 & 1.00 & 0.31 \\
& WISE       & 0.30 & 0.28 & 0.81 & 0.41 & 0.24 & 0.22 & 0.79 & 0.35 & 0.21 & 0.20 & 0.76 & 0.32 \\
& \textit{HoReN}      & \textbf{1.00} & \textbf{0.92} & \textbf{1.00} & \textbf{0.97} & \textbf{0.99} & \textbf{0.89} & \textbf{1.00} & \textbf{0.96} & \textbf{0.99} & \textbf{0.89} & \textbf{0.99} & \textbf{0.96} \\
\midrule
\multirow{3}{*}{\begin{tabular}[c]{@{}l@{}}DeepSeek\\8B\end{tabular}}
& GRACE      & 1.00 & 0.02 & 1.00 & 0.27 & 1.00 & 0.02 & 1.00 & 0.27 & 1.00 & 0.02 & 1.00 & 0.27 \\
& WISE       & 0.46 & 0.44 & 1.00 & 0.59 & 0.36 & 0.33 & 1.00 & 0.49 & 0.31 & 0.30 & 1.00 & 0.45 \\
& \textit{HoReN}      & \textbf{1.00} & \textbf{0.95} & \textbf{0.98} & \textbf{0.98} & \textbf{1.00} & \textbf{0.94} & \textbf{0.97} & \textbf{0.97} & \textbf{1.00} & \textbf{0.93} & \textbf{0.97} & \textbf{0.97} \\
\midrule
\multirow{3}{*}{\begin{tabular}[c]{@{}l@{}}GPT-OSS\\20B\end{tabular}}
& GRACE      & 0.61 & 0.01 & 1.00 & 0.18 & 0.60 & 0.02 & 1.00 & 0.23 & 0.60 & 0.02 & 1.00 & 0.23 \\
& WISE       & 0.14 & 0.14 & 0.08 & 0.12 & 0.12 & 0.11 & 0.12 & 0.12 & 0.12 & 0.11 & 0.12 & 0.12 \\
& \textit{HoReN}      & \textbf{0.99} & \textbf{0.69} & \textbf{1.00} & \textbf{0.88} & \textbf{0.99} & \textbf{0.67} & \textbf{0.99} & \textbf{0.87} & \textbf{0.99} & \textbf{0.66} & \textbf{0.98} & \textbf{0.86} \\
\midrule
\multirow{3}{*}{\begin{tabular}[c]{@{}l@{}}SEA-LION\\32B\end{tabular}}
& GRACE  & 0.43 & 0.02 & 1.00 & 0.20 & 0.43 & 0.02 & 1.00 & 0.20 & 0.43 & 0.02 & 1.00 & 0.20 \\
& WISE   & 0.64 & 0.59 & 1.00 & 0.72 & 0.58 & 0.55 & 1.00 & 0.68 & 0.57 & 0.54 & 1.00 & 0.68 \\
& \textit{HoReN} & \textbf{1.00} & \textbf{0.91} & \textbf{0.98} & \textbf{0.96} & \textbf{1.00} & \textbf{0.90} & \textbf{0.97} & \textbf{0.96} & \textbf{1.00} & \textbf{0.90} & \textbf{0.97} & \textbf{0.96} \\
\bottomrule
\end{tabular}
\end{table*}

Two failure shapes are diagnostic in Tables~\ref{tab:scaling-10k}--\ref{tab:llama-scaling} and Figure~\ref{fig:50k-scaling}. On LLaMA-3.1-8B, AlphaEdit cliffs (OP $0.82 \to 0.10$ from $N{=}2000$ to $5000$ as Locality falls $0.65 \to 0.03$) because a fixed null-space projector cannot include directions added by later edits, while UltraEdit slopes (OP $0.73 \to 0.67$ over $N{=}2000$--$10000$, Locality $0.53 \to 0.42$). On Qwen2.5-7B, DeepSeek-R1-8B, GPT-OSS-20B, and Qwen-SEA-LION-v4-32B (Table~\ref{tab:scaling-10k}), GRACE's Generalization stays in $[0.01, 0.03]$ across scales (the codebook absorbs $10$K entries, so the failure is geometric rather than capacity-bound) and WISE's OP decays monotonically (collapsing to $\leq 0.13$ on GPT-OSS-20B), localizing its weakness to a drifting side-memory whose performance on preserving edits and generalization degrades with scale. HoReN is flat across $N$ in reliability $0.99$--$1.00$, generalization $0.89$--$0.95$ excluding the GPT-OSS-20B model, and locality $0.97$--$1.00$ across all four models, and all three metrics are above $0.89$ at $50$K on LLaMA-3.1-8B, because storage is per-entry and routing depends only on the directional geometry of stored keys, so adding entries does not move existing decision boundaries. Note that HoReN's generalization on GPT-OSS-20B is comparatively lower. Since GPT-OSS-20B employs the MoE structure, HoReN is applied on an attention layer, not an MLP layer with knowledge keys, resulting in the degraded generalization performance. HoReN demonstrates stable performance on reliability, generalization, and locality up to 50K edits on the ZsRE benchmark.

\subsection{Generalization Across Datasets and Models}
\label{sec:expii-generalization}

ZsRE provides a favorable separation regime; we now test whether
HoReN's advantage holds under structurally different conditions along
two orthogonal dimensions: dataset type
(Section~\ref{sec:expii-cross-dataset}) and model family
(Section~\ref{sec:expii-cross-model}).

\subsubsection{Cross-Dataset Generalization}
\label{sec:expii-cross-dataset}

\begin{figure}[t]
    \centering
    \includegraphics[width=0.95\textwidth]{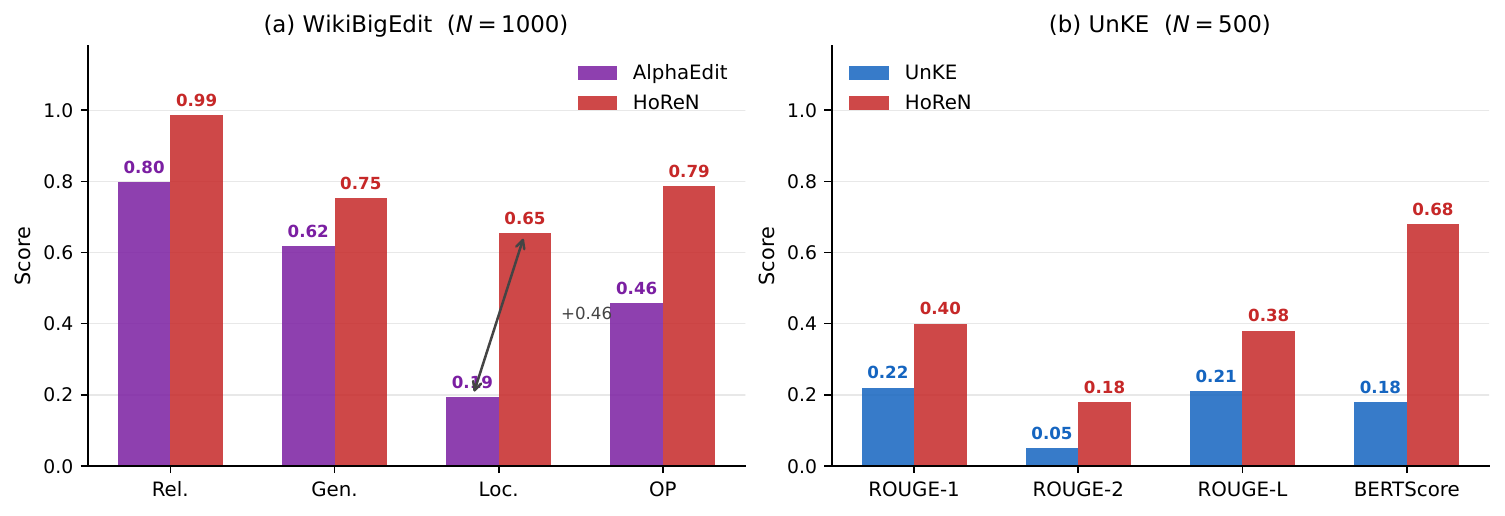}
    \caption{Cross-dataset generalization (LLaMA-3.1-8B). \textbf{(a) WikiBigEdit} ($N{=}1000$): HoReN outperforms AlphaEdit on all three metrics; AlphaEdit's locality collapses to $0.19$ while HoReN retains $0.65$, a gap of $+0.46$. \textbf{(b) UnKE} ($N{=}500$): The four scores in panel (b) are the \emph{Original}-split values from UnKEBench~\citep{Deng2024UnKEUK} and therefore quantify reliability on the original edit question; the complementary \emph{Paraphrase} (generalization) and \emph{Sub-question} (factual coverage) splits, on which HoReN also leads at $N{=}500$, are reported in Table~\ref{tab:unke-breakdown}. Full results across all $N$ in Tables~\ref{tab:alphaedit}--\ref{tab:cross-dataset} (Appendix~\ref{app:cross-dataset-tables}).}
    \label{fig:cross-dataset}
\end{figure}

On WikiBigEdit, where locality queries share subjects and relations with edits, HoReN sustains Reliability $0.986$--$0.992$, Generalization $0.753$--$0.767$, and Locality $0.631$--$0.682$ across $N{\in}\{500,1000,3000\}$ while AlphaEdit's Locality collapses to $0.169$--$0.267$ in the same grid; HoReN still outperforms AlphaEdit on this benchmark. However, the editing samples from WikiBigEdit do not have clear (subject, relation, object) structures as in ZsRE, meaning there is considerable noise besides the knowledge representation at the layer being edited by HoReN, resulting in degraded performance in generalization and locality.  On UnKE, whose targets are multi-sentence free-form passages, HoReN at $N{=}500$ leads the domain-specific UnKE baseline by $+0.18$ ROUGE-1 ($0.40$ vs.\ $0.22$), $+0.13$ ROUGE-2, $+0.17$ ROUGE-L, and $+0.50$ BERTScore on the \emph{Original} split, and also leads on \emph{Paraphrase} (ROUGE-L $0.338$ vs.\ $0.197$, BERTScore $0.659$ vs.\ $0.166$) and \emph{Sub-question} (ROUGE-L $0.332$ vs.\ $0.137$, ROUGE-1 $0.338$ vs.\ $0.141$) splits at $N{=}500$ (Table~\ref{tab:unke-breakdown}), showing the advantage of HoReN's routing mechanism on unstructured data over the UnKE editor.

\subsubsection{Cross-Model Generalization}
\label{sec:expii-cross-model}

Across Qwen2.5-7B, DeepSeek-R1-Distill-Llama-8B, GPT-OSS-20B, and Qwen-SEA-LION-v4-32B at $N{\in}\{2000,5000,10000\}$ (Table~\ref{tab:scaling-10k}), HoReN reaches OP $0.86$--$0.98$ on every cell (Rel.\ $\geq 0.99$, Gen.\ $\geq 0.66$, Loc.\ $\geq 0.97$), while GRACE is stuck at OP $0.18$--$0.31$ with Generalization $0.01$--$0.03$ and WISE ranges from OP $\leq 0.13$ on GPT-OSS-20B to $0.68$--$0.72$ on SEA-LION-32B. The four backbones come from independent training pipelines with different tokenizers, pretraining corpora, and post-training procedures, so GRACE's near-identical flat-Generalization failure across all four (a span of $0.02$) confirms the bottleneck is in the key-query comparison rather than in any model-specific representation property, and WISE's family-dependent collapse localizes its weakness to model-specific activation structure rather than to scale alone. HoReN's recipe is held fixed at $\beta{=}20$, $\gamma{=}0.1$, $M{=}1$ across all backbones; only the matching threshold and token pooling ratio change, listed once per family in Table~\ref{tab:model-params}, a single recipe transferring across seven backbones from $1.5$B to $32$B parameters reflects the directional structure of post-nonlinearity activations discussed in Section~\ref{sec:expii-analysis-ablation}.

\subsection{Why HoReN Works: Direction Is the Semantic Key}
\label{sec:expii-analysis-ablation}

\textbf{Direction, not magnitude, carries the fact.}
ROME framed MLP down-projections as a key--value associative memory, with the post-nonlinearity activation at the subject token serving as the \emph{key} into a linear store of facts and editing as the insertion of one $(k,v)$ pair without disturbing the rest. GRACE adopts the codebook-of-edits paradigm but is deliberately layer-agnostic, attaching to any chosen layer and keying on the cached input activation $h^{l-1}$ matched by Euclidean distance; HoReN takes one further step and recommends placing the codebook at the MLP, where ROME's analysis identifies the activations that actually behave as fact keys. Table~\ref{tab:grace-vs-horen} then probes whether magnitude or direction of these MLP activations carries the fact identity: holding codebook location, value adaptor, and matching threshold fixed and varying only the key--query comparison on LLaMA-3.1-8B (ZsRE, $N{=}1000$), unit-norm cosine matching alone lifts Generalization from $0.07$ to $0.33$ ($+0.26$) without measurable cost to Reliability or Locality (Loc.\ $\geq 0.99$). Stripping magnitude is therefore informative: post-nonlinearity activation magnitude behaves like a per-prompt gain (driven by surface form, length, and upstream normalization), while the \emph{direction} on the unit hypersphere carries the identity of the semantic.
\begin{wrapfigure}{r}{0.45\textwidth}
    \vspace{-6pt}
    \centering
    \includegraphics[width=0.42\textwidth]{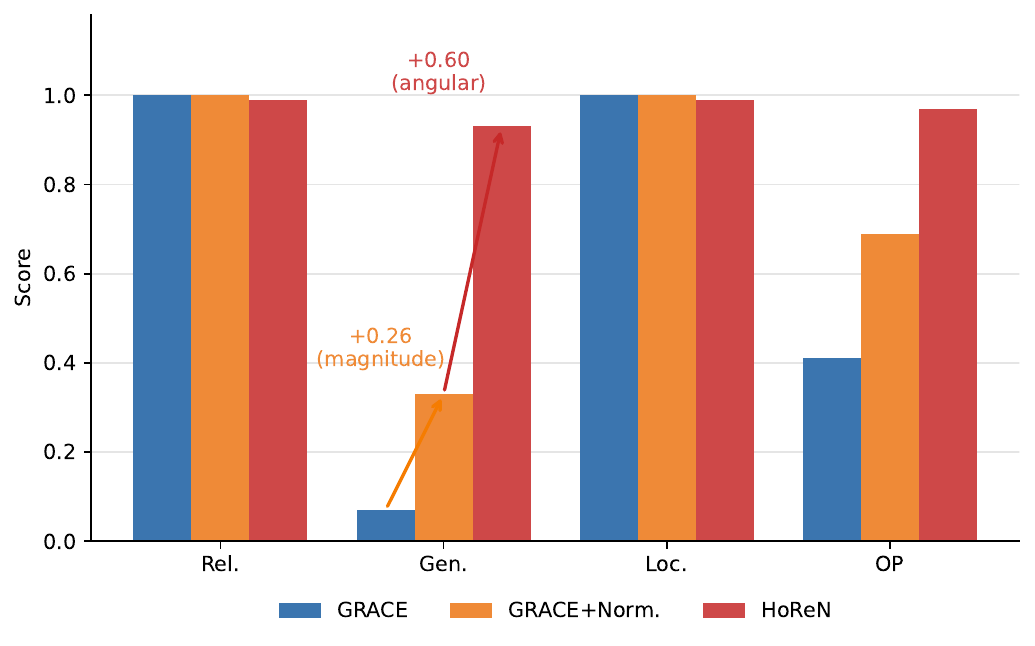}
    \caption{Representation gap diagnosis at $N{=}1000$ (ZsRE, LLaMA-3.1-8B). Generalization improves in two steps: normalization removes the magnitude component ($+0.26$); the Hopfield update closes the angular gap ($+0.60$). Reliability and Locality remain flat, confirming orthogonality of the two ingredients.}
    \label{fig:representation-gap}
    \vspace{-8pt}
\end{wrapfigure}

\textbf{Hopfield dynamics on the unit sphere control the generalization--locality tradeoff.}
Cosine matching alone leaves a residual angular gap: paraphrase queries still sit off-axis from their target keys, and one-shot argmax misroutes whenever that angle exceeds the matching threshold $c$. A damped softmax-attention update over the unit-norm codebook closes the gap, lifting Generalization from $0.33$ to $0.93$ in Table~\ref{tab:grace-vs-horen} (Figure~\ref{fig:representation-gap} visualizes the two-step ladder), and the update is direction-selective rather than uniformly attractive: an unrelated query is roughly equidistant from all stored keys, the softmax is near-uniform, and the update barely moves it; a paraphrase query is close to one key, the softmax is sharply peaked, and the update nudges it into the basin. The selectivity requires both ingredients jointly, as Table~\ref{tab:normalized-codebook} (Appendix~\ref{app:grace-vs-horen}) shows that Hopfield refinement \emph{without} normalization collapses Locality to $\leq 0.08$, since unbounded magnitudes drag unrelated queries into spurious basins. Theorem~\ref{thm:hopfield_asymptotic_convergence} characterizes the asymptotic behavior of these dynamics; HoReN's normalized, lightly damped deployment sits well inside the jointly stable Generalization/Locality region of Figure~\ref{fig:hopfield-steps}.

\textbf{Direction is a stable property of the base model.}
Together these reads explain Tables~\ref{tab:zsre-comparison} and~\ref{tab:scaling-10k}: the codebook plus value adaptor delivers Reliability $\geq 0.99$ on every row, while the routing layer sustains Generalization and Locality at scale by reading only direction. The cross-model parameter bands in Table~\ref{tab:model-params} provide an independent confirmation: matching thresholds cluster in two tight, family-coherent bands (LLaMA-3 / LLaMA-3.1 at $0.80$--$0.85$, Qwen-2.5 / SEA-LION at $0.55$--$0.60$) that survive heavy reasoning post-training, while the rest of the recipe ($\beta{=}20$, $\gamma{=}0.1$, $M{=}1$) is held fixed across all seven backbones from $1.5$B to $32$B parameters. Token-pooling and value-adaptor (direct vs.\ LoRA) ablations are deferred to Appendices~\ref{app:token-selection} and~\ref{app:lora-vs-value}.

\section{Conclusion}
\label{sec:conclusion}
We studied lifelong model editing through the lens of memory-based editors and identified paraphrase routing under accumulated edits as the bottleneck of this paradigm: existing methods can store edits faithfully but cannot reliably activate the right one when the codebook grows. We proposed \textbf{HoReN}, a routing-layer mechanism that combines L2-normalized angular matching with a single damped Hopfield-style query refinement, motivated by an attractor-basin reading in which the deferral threshold serves as the basin boundary. Theoretically, we showed that standard Hopfield retrieval admits convergent attractor dynamics, which is monotone energy descent with convergence to fixed points and would attract all queries, including unrelated ones, toward stored codes if iterated to convergence, justifying HoReN's single-step deployment. Empirically, HoReN sustains all three editing metrics above $0.89$ on a controlled $50$K-edit ZsRE stress test, reaches OP $0.94$--$0.97$ at $1{,}000$ edits across LLaMA and Qwen backbones, generalizes to structured WikiBigEdit and unstructured UnKE, and matches GRACE on edit/inference latency and parameter overhead. Limitations and future directions, linear side-memory growth, sensitivity to prompt structure, the asymptotic--single-step theory gap, and out-of-scope edit interactions, are discussed in Appendix~\ref{app:limitations}.

\bibliographystyle{plainnat}
\bibliography{example_paper}


\newpage
\appendix
\section{Limitations and Future Directions}
\label{app:limitations}

Four limitations frame the next steps for HoReN. \emph{(i) Linear side-memory growth.} HoReN's codebook grows linearly in the number of edits at ${\sim}4.08$M parameters per $1$K edits (Appendix~\ref{app:efficiency}); although it dominates parameter-modifying baselines and Pareto-dominates WISE on the (memory, generalization) frontier up to $N\!\approx\!14$--$15$K edits, beyond that point WISE's fixed-size dual memory becomes asymptotically smaller. A natural remedy is \emph{codebook compression}, which could be implemented by product/residual quantization of keys, low-rank or shared-basis pooling of value adaptors, and learned merging of semantically adjacent edits, so that the effective per-edit memory shrinks with $N$ while preserving routing fidelity. \emph{(ii) Sensitivity to prompt structure.} HoReN reaches its high-water mark on ZsRE, where each prompt cleanly factorizes into (subject, relation, object); on WikiBigEdit, whose prompts carry richer syntax and discourse (multiple entities, temporal qualifiers, longer contexts), HoReN still outperforms all baselines (Section~\ref{sec:expii-cross-dataset}) but its absolute scores drop, indicating that single-token, single-layer keying is an under-specification for complex prompts. Promising directions include multi-token / multi-layer key construction, structure-aware pooling (e.g., dependency- or entity-anchored), and routing over multiple candidate keys per edit so that reordered or compositional paraphrases remain reachable. \emph{(iii) Theory--deployment gap.} Theorem~\ref{thm:hopfield_asymptotic_convergence} characterizes asymptotic multi-step Hopfield convergence, whereas HoReN deploys a single damped step ($M{=}1$, $\gamma{=}0.1$); a finite-step, finite-codebook analysis bounding routing error after one damped pass as a function of $\beta$, $\gamma$, and codebook geometry would tighten the theory--practice link. \emph{(iv) Out-of-scope edit interactions.} The current single-key argmax routing treats edits as mutually independent and uses a fixed per-model threshold $c$; extending HoReN to temporal conflicts among edits, multi-hop reasoning across edits, and adaptive (per-query or per-region) matching thresholds are natural next steps.

\section{Related Work}
\label{app:related-work}

\paragraph{Parameter-modifying editors.} Locate-and-edit methods such as ROME~\citep{Meng2022LocatingAE} and MEMIT~\citep{Meng2022MassEditingMI} modify internal weights via causal tracing or rank-constrained updates but suffer interference under sequential application, a variant of the catastrophic forgetting problem long studied in continual learning~\citep{kirkpatrick2017overcoming}. AlphaEdit~\citep{Fang2024AlphaEditNC} mitigates this with null-space projection, and UltraEdit~\citep{Gu2025UltraEditTS} studies training-free unstructured editing. These methods concern how edits are \emph{written} into the base weights; HoReN concerns how accumulated edits are \emph{routed} at inference time.

\paragraph{Parameter-preserving / memory-based editors.} Memory-based editors store updates externally and activate them on matching queries, with each method innovating on a different axis: DEFER~\citep{mitchell2022memory} learns a refinement network for retrieval, GRACE~\citep{hartvigsen2023aging} maintains a discrete codebook keyed on hidden states with a fixed $\varepsilon$-radius match, WISE~\citep{Wang2024WISERT} adds activation-based importance weighting, and REPAIR~\citep{wang2025repair} enforces cross-layer consistency. Value payloads are typically lightweight modules, LoRA~\citep{hu2022lora}, adapters~\citep{Houlsby2019ParameterEfficientTL}, or prefix tuning~\citep{Li2021PrefixTuningOC}, recently surveyed by~\citet{Zhang2024ACS}. Conceptually this family is close to retrieval-augmented language modeling~\citep{Khandelwal2019GeneralizationTM,lewis2020retrieval}, but with keys and values constructed per edit. HoReN stays in this family and differs from GRACE along two specific axes: (i) it L2-normalizes keys and queries so matching depends on angular similarity rather than magnitude, and (ii) it applies a single damped Hopfield-style refinement step on the query before the final argmax; value adaptors are kept modular (direct vector or LoRA).

\paragraph{Associative retrieval with Hopfield networks.} Modern Hopfield networks~\citep{Ramsauer2020HopfieldNI} extend classical associative memory with continuous energy functions and high storage capacity, retrieving patterns through attractor dynamics that pull noisy queries toward stored memories, with connections to attention and few-shot learning~\citep{Krotov2023ANF}. To our knowledge, this perspective has not been brought to large-scale sequential model editing; HoReN borrows the energy/attractor view to motivate a lightweight retrieval refinement rather than to deploy full Hopfield dynamics.

\section{Method Details}
\label{app:method-details}

\subsection{Algorithm Pseudocode}
\label{app:algorithm}

\begin{algorithm}[H]
\caption{Codebook Update at Layer $l$}
\label{alg:hopfield_edit_simple}
\KwIn{Codebook $\mathcal{C}=\{(k_i, v_i, y_i)\}_{i=1}^{|C|}$, possibly empty. Model with codebook params $f_{\theta_C}$.}
\KwIn{Edit sample $(\mathbf{x}_t,\mathbf{y}_t^*)$}
\KwIn{Hopfield params $(\gamma,\beta, \epsilon)$, iterations $M$, match threshold $c$, adaptor steps $U$}
\KwOut{Updated codebook $\mathcal{C}$}

\BlankLine
$(H_{\mathrm{in}},\, H_{\mathrm{out}}) \leftarrow f_\theta(\mathbf{x}_t)$ \tcp*{$H_{\mathrm{in}}$: layer $l$ input , $H_{\mathrm{out}}$: layer $l$ output}
$q_0 \leftarrow \operatorname{normalize}(\operatorname{select}(H_{\mathrm{in}}))$\;
$q \leftarrow q_0$\;

\BlankLine
\If{$\mathcal{C}=\emptyset$}{
    $v_{\text{new}} \leftarrow \textsc{InitValueAdaptor}()$\;
    \For{$u \leftarrow 1$ \KwTo $U$}{
        $H_{\mathrm{out}} \leftarrow H_{\mathrm{out}} + v_{\text{new}}$\;
        Compute $\mathcal{L}_{\mathrm{CE}}(f_{\theta_C}(\mathbf{x}_t), \mathbf{y}_t^*)$ and update only $v_{\text{new}}$ by gradient descent\;
        \If{early stopping criterion is met}{
            \textbf{break}\;
        }
    }
    $\mathcal{C} \leftarrow \mathcal{C} \cup \{(q_0, v_{\text{new}}, \mathbf{y}_t^*)\}$\;
    \Return $\mathcal{C}$\;
}

\BlankLine
$K \leftarrow \operatorname{stack}(\{k_i:(k_i,v_i,y_i)\in\mathcal{C}\})$\;
\For{$m \leftarrow 1$ \KwTo $M$}{
    $p \leftarrow \operatorname{softmax}(\beta \, q K^\top)$\;
    $q_{\text{new}} \leftarrow \operatorname{normalize}(pK)$\;
    \If{$||q_{\text{new}}-q||_2 \le \epsilon$} {
     \tcp{Early Stop}
     \textbf{break}\;
    }
    $q \leftarrow \operatorname{normalize}\!\big((1-\gamma)q + \gamma q_{\text{new}}\big)$\;
}

\BlankLine
$\mathbf{a} \leftarrow qK^\top$\;
$i^\star \leftarrow \arg\max_i \mathbf{a}_i$\;

\BlankLine
\If{$\mathbf{a}_{i^\star} \leq c$ \textbf{or} $y_{i^\star} \neq \mathbf{y}_t^*$}{
    \tcp{No match or label conflict: add a new entry}
    $k_{\text{new}} \leftarrow q_0$\;
    $v_{\text{new}} \leftarrow \textsc{InitValueAdaptor}()$\;
    \For{$u \leftarrow 1$ \KwTo $U$}{
        $H_{\mathrm{out}} \leftarrow H_{\mathrm{out}} + v_{\text{new}}$\;
        Compute $\mathcal{L}_{\mathrm{CE}}(f_{\theta_C}(\mathbf{x}_t), \mathbf{y}_t^*)$ and update only $v_{\text{new}}$ by gradient descent\;
        \If{early stopping criterion is met}{
            \textbf{break}\;
        }
    }
    $\mathcal{C} \leftarrow \mathcal{C} \cup \{(k_{\text{new}}, v_{\text{new}}, \mathbf{y}_t^*)\}$\;
}
\Else{
    \tcp{Match with consistent label: use existing entry}
    \For{$u \leftarrow 1$ \KwTo $U$}{
        $H_{\mathrm{out}} \leftarrow H_{\mathrm{out}} + v_{i^\star}$\;
        Compute $\mathcal{L}_{\mathrm{CE}}(f_{\theta_C}(\mathbf{x}_t), \mathbf{y}_t^*)$ and update only $v_{i^\star}$ by gradient descent\;
        \If{early stopping criterion is met}{
            \textbf{break}\;
        }
    }   
}

\Return $\mathcal{C}$
\end{algorithm}

\section{Proofs of Theoretical Results}
\label{sec:proofs}

\subsection{Setup and Descent Inequality}

The two results below share a common Lyapunov construction. We isolate the per-step descent inequality first; the asymptotic theorem and finite-step proposition then follow as two complementary readings of it.

\begin{lemma}[Per-step descent of standard Hopfield retrieval]
\label{lem:hopfield_descent}
Let $K\in\mathbb{R}^{C\times d}$ be a codebook with rows satisfying $\|k_i\|_2=1$, let $q^{(0)}\in\mathbb{R}^{1\times d}$ satisfy $\|q^{(0)}\|_2=1$, and let $\beta>0$. Define
\[
T(q):=\operatorname{softmax}(\beta qK^\top)\,K,\quad
F(q):=\tfrac1\beta\log\!\sum_{i=1}^{C}\exp(\beta\,qk_i^\top),\quad
E(q,K):=\tfrac12\|q\|_2^2-F(q).
\]
Then for every $s\ge 0$ the iterate $q^{(s+1)}=T(q^{(s)})$ satisfies
\[
E(q^{(s+1)},K)-E(q^{(s)},K)\;\le\;-\tfrac12\,\|q^{(s+1)}-q^{(s)}\|_2^2.
\]
\end{lemma}

\begin{proof}
We construct a Lyapunov function. For $q\in\mathbb{R}^{1\times d}$, the potential $F$ is the log-sum-exp of the similarities,
\begin{equation}
F(q) := \frac{1}{\beta} \log \left( \sum_{i=1}^C \exp(\beta \, q k_i^\top) \right) = \frac{1}{\beta} \text{lse}_\beta(q K^\top),
\end{equation}
where $qk_i^\top$ denotes the dot product between the query and the $i$-th key. The function $F(q)$ is convex in $q$. Its gradient with respect to $q$ is:
\begin{equation}
\nabla_q F(q) = \frac{\sum_{i=1}^C k_i \exp(\beta \, q k_i^\top)}{\sum_{i=1}^C \exp(\beta \, q k_i^\top)} = \text{softmax}(\beta \, q K^\top) K = p K,
\end{equation}
where $p = \text{softmax}(\beta \, q K^\top) \in \mathbb{R}^{1 \times C}$. Consequently, the standard Hopfield update corresponds exactly to the gradient mapping:
\begin{equation} \label{eq:update_app}
q^{(s+1)} = \nabla_q F(q^{(s)}).
\end{equation}

Define the energy function as $E(q, K) := \frac{1}{2}\|q\|_2^2 - F(q)$. Since $F$ is convex, it satisfies the first-order inequality:
\begin{equation}
F(y) \ge F(x) + \langle \nabla_q F(x), (y-x)^\top \rangle
\end{equation}
for all $x, y \in \mathbb{R}^{1 \times d}$. Letting $x=q^{(s)}$ and $y=q^{(s+1)}$, and using \eqref{eq:update_app}, we obtain:
\begin{equation}
F(q^{(s+1)}) - F(q^{(s)}) \ge \langle q^{(s+1)}, (q^{(s+1)} - q^{(s)})^\top \rangle = q^{(s+1)} (q^{(s+1)} - q^{(s)})^\top.
\end{equation}

We now evaluate the change in energy between consecutive iterations:
\begin{equation}
\begin{aligned}
E(q^{(s+1)}, K) - E(q^{(s)}, K)
&= \frac{1}{2}\|q^{(s+1)}\|_2^2 - \frac{1}{2}\|q^{(s)}\|_2^2 - \Big( F(q^{(s+1)}) - F(q^{(s)}) \Big) \\
&\leq \frac{1}{2}\|q^{(s+1)}\|_2^2 - \frac{1}{2}\|q^{(s)}\|_2^2 - q^{(s+1)} (q^{(s+1)} - q^{(s)})^\top \\
&= -\frac{1}{2} \|q^{(s+1)} - q^{(s)}\|_2^2 \leq 0.
\end{aligned}
\end{equation}

This establishes the descent inequality.\qed
\end{proof}

The descent inequality has two complementary readings. \emph{Asymptotically} ($s\to\infty$), it forces $\|q^{(s+1)}-q^{(s)}\|_2\to 0$ and, combined with real-analyticity of $E$, gives single-limit convergence to a fixed point of $T$. \emph{Finitely} (any $M\ge 1$), it gives a cumulative bound on the squared increments, hence a $1/\sqrt{M}$ residual rate for the best of the first $M$ iterates (Proposition~\ref{prop:hopfield_finite_step}). We state and prove the two results separately below.

\subsection{Proof of Theorem~\ref{thm:hopfield_asymptotic_convergence}: Asymptotic Convergence}

\noindent\textbf{Theorem~\ref{thm:hopfield_asymptotic_convergence} (restated).} \emph{Under the assumptions above, the sequence $\{q^{(s)}\}_{s=0}^{\infty}$ generated by $q^{(s+1)}=T(q^{(s)})$ converges to a fixed point $q^\ast$, i.e., $q^{(s)}\to q^\ast$ as $s\to\infty$, where $q^\ast=T(q^\ast)=\operatorname{softmax}(\beta q^\ast K^\top)\,K$.}

\begin{proof}
\emph{Step 1: Boundedness of iterates and lower-boundedness of $E$.} Since $q^{(s+1)}=\operatorname{softmax}(\beta q^{(s)}K^\top)K$ is a convex combination of unit-norm rows of $K$, $q^{(s+1)}\in\operatorname{conv}\{k_1,\ldots,k_C\}$ and $\|q^{(s+1)}\|_2\le 1$. Together with $\|q^{(0)}\|_2=1$, the full sequence is bounded. For any $q$,
\[
F(q)=\tfrac1\beta\log\!\sum_{i=1}^{C}\exp(\beta\,qk_i^\top)\le \max_i qk_i^\top+\frac{\log C}{\beta}\le \|q\|_2+\frac{\log C}{\beta},
\]
so $E(q,K)\ge\tfrac12\|q\|_2^2-\|q\|_2-\tfrac{\log C}{\beta}\ge -\tfrac12-\tfrac{\log C}{\beta}$. Thus $E$ is bounded below.

\emph{Step 2: Square-summability of increments.} Set $E_s:=E(q^{(s)},K)$ and $d_s:=q^{(s+1)}-q^{(s)}$. Lemma~\ref{lem:hopfield_descent} gives $E_s-E_{s+1}\ge\tfrac12\|d_s\|_2^2$. Since $\{E_s\}$ is monotone and bounded below, it converges to some $E_\infty$, and
\[
\sum_{s=0}^{\infty}\|d_s\|_2^2\le 2(E_0-E_\infty)<\infty,
\]
so $\|d_s\|_2\to 0$.

\emph{Step 3: Accumulation points are fixed points.} Let $\Omega$ be the set of accumulation points of $\{q^{(s)}\}$; boundedness gives $\Omega\neq\emptyset$. For $q^\ast\in\Omega$, pick a subsequence $q^{(s_j)}\to q^\ast$. Since $\|d_{s_j}\|_2\to 0$, $q^{(s_j+1)}\to q^\ast$, and $q^{(s_j+1)}=T(q^{(s_j)})\to T(q^\ast)$ by continuity of $T$. Hence $q^\ast=T(q^\ast)$.

\emph{Step 4: Single-limit convergence via Kurdyka-\L{}ojasiewicz.} Since $F$ is real analytic, so is $E$, and $E$ satisfies the Kurdyka-\L{}ojasiewicz (KL) property. Note that $\nabla E(q)=q-\nabla F(q)=q-T(q)$, so $\nabla E(q^{(s)})=q^{(s)}-q^{(s+1)}=-d_s$. Since $E_s\to E_\infty$, every $q^\ast\in\Omega$ satisfies $E(q^\ast,K)=E_\infty$. By the uniformized KL property on the compact set $\Omega$, there exist $\eta>0$, $\varepsilon>0$, and a concave desingularizing $\varphi$ with $\varphi(0)=0$, $\varphi'>0$, such that for all sufficiently large $s$ with $E_s>E_\infty$,
\[
\varphi'(E_s-E_\infty)\,\|\nabla E(q^{(s)})\|_2\ge 1.
\]
(If $E_s=E_\infty$ for some $s$, the descent inequality forces $d_s=0$ and the sequence is constant from that point on.) By concavity of $\varphi$,
\[
\begin{aligned}
\varphi(E_s-E_\infty)-\varphi(E_{s+1}-E_\infty)
&\ge\varphi'(E_s-E_\infty)(E_s-E_{s+1})\\
&\ge\frac{E_s-E_{s+1}}{\|\nabla E(q^{(s)})\|_2}=\frac{E_s-E_{s+1}}{\|d_s\|_2}\ge\tfrac12\|d_s\|_2.
\end{aligned}
\]
Hence $\|d_s\|_2\le 2\bigl(\varphi(E_s-E_\infty)-\varphi(E_{s+1}-E_\infty)\bigr)$, and summing gives $\sum_{s=0}^{\infty}\|d_s\|_2<\infty$. Therefore $\{q^{(s)}\}$ has finite length, is Cauchy, and converges to some $q^\ast$. Combined with Step~3, $q^\ast=T(q^\ast)=\operatorname{softmax}(\beta q^\ast K^\top)\,K$. \qed
\end{proof}

\subsection{Proof of Proposition~\ref{prop:hopfield_finite_step}: Finite-Step Descent and Residual Bound}

\noindent\textbf{Proposition~\ref{prop:hopfield_finite_step} (restated).} \emph{Under the assumptions of Theorem~\ref{thm:hopfield_asymptotic_convergence}, let $M\ge 1$ be a finite integer and generate $q^{(s+1)}=T(q^{(s)})$, $s=0,1,\ldots,M-1$. Then, for every $s$,
$E(q^{(s+1)},K)\le E(q^{(s)},K)-\tfrac12\|q^{(s+1)}-q^{(s)}\|_2^2,$
and consequently
$\sum_{s=0}^{M-1}\|q^{(s+1)}-q^{(s)}\|_2^2 \le 2\bigl(E(q^{(0)},K)-E(q^{(M)},K)\bigr).$
Let $r_s:=\|T(q^{(s)})-q^{(s)}\|_2=\|q^{(s+1)}-q^{(s)}\|_2$ be the fixed-point residual at step $s$. For any lower bound $E_{\inf}$ of $E$,
$\frac{1}{M}\sum_{s=0}^{M-1}r_s^2 \le \frac{2(E(q^{(0)},K)-E_{\inf})}{M}$ and $\min_{0\le s<M}r_s \le \sqrt{2(E(q^{(0)},K)-E_{\inf})/M}.$
With the explicit constants $E_{\inf}=-\tfrac12-\frac{\log C}{\beta}$ and the normalization $\|q^{(0)}\|_2=1$, $\|k_i\|_2=1$, this yields $\min_{0\le s<M}\|T(q^{(s)})-q^{(s)}\|_2 \le 2/\sqrt{M}.$}

\begin{proof}
The per-step descent inequality is exactly Lemma~\ref{lem:hopfield_descent}. Summing it from $s=0$ to $M-1$ gives
\[
E(q^{(M)},K)-E(q^{(0)},K)\le -\tfrac12\sum_{s=0}^{M-1}\|q^{(s+1)}-q^{(s)}\|_2^2,
\]
or equivalently $\sum_{s=0}^{M-1}r_s^2\le 2\bigl(E(q^{(0)},K)-E(q^{(M)},K)\bigr)$.

If $E_{\inf}\le E(q,K)$ for all $q$, then $E(q^{(M)},K)\ge E_{\inf}$, so
\[
\sum_{s=0}^{M-1}r_s^2\;\le\;2\bigl(E(q^{(0)},K)-E_{\inf}\bigr).
\]
Dividing by $M$ and using $\min_{s} r_s^2\le\tfrac1M\sum_s r_s^2$ yields the residual bound.

For the explicit constant, Step~1 of the proof of Theorem~\ref{thm:hopfield_asymptotic_convergence} gives $E(q,K)\ge -\tfrac12-\tfrac{\log C}{\beta}$, so we may take $E_{\inf}=-\tfrac12-\tfrac{\log C}{\beta}$. Under the normalization $\|q^{(0)}\|_2=1$, $\|k_i\|_2=1$, $q^{(0)}k_i^\top\ge -1$ for every $i$, hence
\[
\sum_{i=1}^{C}\exp(\beta\, q^{(0)}k_i^\top)\ge C e^{-\beta},
\qquad
F(q^{(0)})\ge -1+\frac{\log C}{\beta},
\]
and thus
\[
E(q^{(0)},K)=\tfrac12-F(q^{(0)})\le \tfrac32-\frac{\log C}{\beta}.
\]
Combining with $E_{\inf}=-\tfrac12-\frac{\log C}{\beta}$ gives $E(q^{(0)},K)-E_{\inf}\le 2$, and therefore
\[
\min_{0\le s<M}r_s\;\le\;\sqrt{\frac{2\cdot 2}{M}}=\frac{2}{\sqrt{M}}. \qquad\qed
\]
\end{proof}

\paragraph{Interpretation of the finite-step bound.} Proposition~\ref{prop:hopfield_finite_step} controls the \emph{best} of the first $M$ iterates: there exists $\bar s\in\arg\min_{0\le s<M}\|T(q^{(s)})-q^{(s)}\|_2$ with residual at most $2/\sqrt{M}$. It does \emph{not} certify that the last iterate $q^{(M)}=T^M(q^{(0)})$ is itself an approximate fixed point, only $E(q^{(M)},K)\le E(q^{(0)},K)$ is guaranteed. Whether $\|T(q^{(M)})-q^{(M)}\|_2$ is small at the last step must be checked separately. This distinction matters for HoReN: the algorithm fixes $M=1$ and outputs $q^{(1)}$ rather than the residual-minimizing iterate, so the design relies on the asymptotic-attraction picture of Theorem~\ref{thm:hopfield_asymptotic_convergence} as a warning, namely, that running $M\gg 1$ steps would contract \emph{every} query, including unrelated ones, rather than as a route to a tighter approximate fixed point.

\section{Supplementary Experimental Setup Details}
\label{app:experimental-details}

In this appendix, we provide a detailed description of the experimental configuration, including an introduction to the datasets, a comprehensive explanation of the evaluation metrics, the implementation details, and a discussion of the baselines.

\subsection{Datasets}
\label{app:datasets}

Here, we provide a detailed introduction to the datasets used in this paper:
\begin{itemize}
\item \textbf{ZsRE}~\citep{levy2017zero} is a question-answering dataset that has become the standard benchmark for lifelong knowledge editing. Edit prompts are factual questions paired with target answers, and rephrased prompts are produced via back-translation to evaluate generalization. Following prior work~\citep{Mitchell2021FastME,hartvigsen2023aging}, natural questions unrelated to the edited fact are used as out-of-scope queries to evaluate locality. Each sample contains (i) an editing prompt, (ii) the target answer, (iii) a paraphrased prompt, and (iv) a locality prompt with its own ground-truth answer. We use the split from~\citet{Mitchell2021FastME} (${>}15$K samples); the $50$K stress test is constructed from the same source distribution (Appendix~\ref{app:50k-construction}).
\item \textbf{WikiBigEdit}~\citep{Thede2025WikiBigEditUT} is a large-scale benchmark constructed from Wikidata triples that introduces a substantially harder locality regime: locality queries are sampled from triples sharing the same subject and related relations as the edited facts, placing them semantically close to edit queries. This setting tests whether an editor can preserve clean angular separation between edits and unrelated knowledge even when the two are topically adjacent, a property that is trivially satisfied on ZsRE but fails for many existing methods on WikiBigEdit.
\item \textbf{UnKE}~\citep{Deng2024UnKEUK} is an unstructured knowledge editing benchmark in which edits are free-form passages rather than triples and target answers are multi-sentence generations rather than short tokens. UnKE serves as a probe into the robustness of an editor beyond the structured (subject, relation, object) regime. Because targets are free-form, evaluation uses surface- and semantic-similarity metrics (ROUGE-1/2/L and BERTScore) instead of token-level exact match (see Appendix~\ref{app:metrics-unstructured}).
\end{itemize}

\subsection{Baselines}
\label{app:baselines}

We describe the baselines used across the experiments. The descriptions are grouped by the binary categorization introduced in Section~\ref{sec:expii-setup}: (I) parameter-modifying methods, which write closed-form updates into the base model's feed-forward weights, and (II) parameter-preserving methods, which leave every base feed-forward weight frozen and store edits in an external/side module consulted at inference. The UnKE baseline is included only for the unstructured editing benchmark, as stated in Section~\ref{sec:expii-setup}.

\paragraph{(I) Parameter-modifying baselines.}
\begin{itemize}
\item \textbf{ROME}~\citep{Meng2022LocatingAE}: a single-fact editor for structured (subject, relation, object) triples. ROME first runs a causal-tracing analysis to identify the single mid-layer feed-forward block most responsible for recalling a given fact, then treats that block as an associative key--value memory in which the ``key'' is the internal representation of the subject and the ``value'' is the representation that triggers the object. To install a new fact it computes a closed-form, rank-one adjustment to that block's output projection so that the subject's key is mapped to the new object's representation while the responses for all other previously cached keys are left as unchanged as possible. ROME edits one fact at a time and modifies the base model's weights in place.

\item \textbf{AlphaEdit}~\citep{Fang2024AlphaEditNC}: a closed-form batch editor designed for sequential editing. AlphaEdit performs a MEMIT-style update to a feed-forward block, but before applying the update it first projects the proposed weight change onto the subspace orthogonal to the internal representations of a sample of preserved facts drawn from a general corpus. Because the update is forced to lie in this ``do-not-disturb'' subspace, applying it to the model leaves the outputs for those preserved representations exactly unchanged, which lets many edits be stacked sequentially with reduced interference on previously stored knowledge.

\item \textbf{UltraEdit}~\citep{Gu2025UltraEditTS}: a training-free, subject-free, and memory-free editor. UltraEdit requires no optimizer training, no identification of a subject token, and stores no record of past edits. From a single forward pass it reads the hidden representation at a chosen layer and one gradient signal that points in the direction of the desired output, and combines them in closed form into a one-step adjustment of the layer's weights. As successive edits arrive it maintains a running normalization of the statistics of those hidden representations so that later edits remain well-scaled relative to earlier ones, enabling long edit streams.

\item \textbf{UnKE}~\citep{Deng2024UnKEUK}: a domain-specific baseline for unstructured edits whose targets are free-form multi-sentence passages rather than short entity strings. UnKE extends the locate-and-edit recipe by working at the granularity of an entire layer's input/output rather than a single subject token: the ``key'' becomes the chosen layer's input representation of the whole question and the ``value'' becomes the layer-output representation that, if produced, would cause the rest of the network to generate the desired passage. UnKE first solves a causal optimization to determine that target layer-output, then updates the feed-forward weights of that layer so the question's input representation produces it.
\end{itemize}

\paragraph{(II) Parameter-preserving baselines.}
\begin{itemize}
\item \textbf{GRACE}~\citep{hartvigsen2023aging}: a lifelong editor that leaves all model weights frozen and instead maintains an external lookup table whose entries are pairs of (a stored hidden representation, the replacement hidden representation that should be produced at one chosen layer). Each entry also owns a small ``acceptance radius'' around its stored representation. At inference, the current activation at that layer is compared against every stored entry; if it falls inside any entry's acceptance radius, the activation is overwritten by that entry's replacement, otherwise the model runs unchanged. Entries are added, shrunk, or merged as new edits arrive, so the lookup table grows monotonically with the edit stream.

\item \textbf{WISE}~\citep{Wang2024WISERT}: a lifelong editor that picks one mid-to-late feed-forward block of the base model and never modifies its weights. WISE allocates an extra copy of just that block's value-projection matrix as a ``side memory'', initialized identically to the original. All edits are written into the side copy only. To prevent edits from interfering with each other, the edit history is split into shards: each shard is trained on a different random subset of the side copy's parameters, and the resulting shards are then combined into one side copy via the Ties-merging procedure -- this merging happens across the side shards only and the side copy is never folded back into the original block. At inference, a routing rule measures how strongly each input token activates the difference between the side copy and the original block; tokens whose activation crosses a learned threshold are sent through the side copy, all others go through the untouched original block.

\item \textbf{REPAIR}~\citep{wang2025repair}: a closed-loop extension of the WISE side-memory paradigm. REPAIR keeps the base feed-forward weights frozen and stores edits in $K$ shard copies of one mid-to-late FFN value matrix, routed at inference by an activation-score margin between each shard and the original block. On top of pure side-memory routing it adds three components to stabilize long edit streams: (i) distribution-aware inner-batch knowledge distillation that groups semantically similar edits and aligns their hidden representations before writing, (ii) a closed-loop error-feedback controller that monitors per-shard accuracy after each edit window and prunes-and-retrains the worst-performing shard on the accumulated error pool, and (iii) loss-aware Weighted TIES merging that fuses the shard deltas into a single side memory with weights inversely proportional to each shard's training loss.
\end{itemize}

\subsection{Metrics}
\label{app:metrics}

We evaluate knowledge editing methods along three axes, reliability, generalization, and locality, following the evaluation framework of \citet{Wang2024WISERT}. All metrics are computed over a set of $N$ edit instances $\{(x_i, y_i)\}_{i=1}^N$, where $x_i$ is the edit prompt and $y_i$ is the desired target answer. We denote the original unedited model by $f_{\theta_0}$ and the model after sequentially applying all $N$ edits by $f_{\theta_N}$. Let $\mathcal{V}$ denote the vocabulary. Given a model $f$ and prompt $x$, we obtain a predicted sequence $\hat{y} = (\hat{t}_1, \dots, \hat{t}_L)$ via greedy autoregressive decoding,
\begin{equation}
\hat{t}_j = \operatorname*{arg\,max}_{v \in \mathcal{V}}\; P_f\!\left(v \mid x,\, \hat{t}_1, \dots, \hat{t}_{j-1}\right),
\end{equation}
where $L$ is fixed to the number of tokens in the target $y_i$. We define the per-instance task measure
\begin{equation}
m(y, \hat{y}) = \frac{1}{L}\sum_{j=1}^{L} \mathbf{1}[\hat{t}_j = t_j],
\end{equation}
which computes the fraction of correctly predicted tokens.

\subsubsection{Structured Metrics (ZsRE \& WikiBigEdit)}
\label{app:metrics-structured}

Following the previous work~\citep{Mitchell2021FastME,hartvigsen2023aging,Wang2024WISERT}, the three metrics on structured benchmarks are defined as follows.
\begin{itemize}
\item \textbf{Reliability} measures whether the edited model produces the target answer for the original edit prompt:
\begin{equation}
\mathrm{Reliability} = \frac{1}{N}\sum_{i=1}^{N} m\!\left(y_i,\;\hat{y}_i^{\,f_{\theta_N},\, x_i}\right).
\end{equation}
\item \textbf{Generalization} evaluates whether the edit transfers to semantically equivalent rephrasings $x_i'$ of the edit prompt:
\begin{equation}
\mathrm{Generalization} = \frac{1}{N}\sum_{i=1}^{N} m\!\left(y_i,\;\hat{y}_i^{\,f_{\theta_N},\, x_i'}\right).
\end{equation}
\item \textbf{Locality} quantifies the extent to which the edit preserves the model's behavior on unrelated inputs. For each locality prompt $x_i^\ell$ with associated ground-truth length $L_i^\ell$, we extract the last $L_i^\ell$ tokens from the full output sequence of both the original and edited models, denoted $z_i^{\,f_{\theta_0}}$ and $z_i^{\,f_{\theta_N}}$ respectively, and compute
\begin{equation}
\mathrm{Locality} = \frac{1}{N}\sum_{i=1}^{N} \frac{1}{L_i^\ell}\sum_{j=1}^{L_i^\ell} \mathbf{1}\!\left[z^{\,f_{\theta_0}}_{i,j} = z^{\,f_{\theta_N}}_{i,j}\right].
\end{equation}
Notably, locality compares the original and edited model outputs to each other rather than against a ground-truth label, directly measuring behavioral preservation.
\end{itemize}
The Overall Performance (OP) score reported alongside the three axes is their geometric mean, $\mathrm{OP} = (\mathrm{Reliability}\cdot\mathrm{Generalization}\cdot\mathrm{Locality})^{1/3}$.

\subsubsection{Unstructured Metrics (UnKE)}
\label{app:metrics-unstructured}

Because UnKE targets are multi-sentence free-form passages, the per-token measure $m(y,\hat{y})$ is replaced by surface-level and semantic-level similarity scores between the generated answer $\hat{y}_i$ and the reference $y_i$. Following the original UnKE protocol~\citep{Deng2024UnKEUK}, the main text reports BLEU, ROUGE-1, ROUGE-2, ROUGE-L, and BERTScore. Appendix~\ref{app:cross-dataset-tables} additionally includes the full UnKEBench view with Original, Paraphrase, and Sub-question splits and metric descriptions; these correspond to reliability, generalization, and factual coverage for unstructured passages, rather than to the structured locality metric used for ZsRE and WikiBigEdit.

\subsection{Implementation Details}
\label{app:implementation-details}

\subsubsection{Baseline Hyperparameters}
\label{app:baseline-params}

For every baseline we adopt the original hyperparameters published by the corresponding authors and only re-map architecture-specific module names (e.g., \texttt{mlp.down\_proj} for LLaMA/Qwen, \texttt{self\_attn.o\_proj} for GPT-OSS) when porting a recipe to a model that the original paper did not study. Random seeds are fixed to $42$ across all experiments for all baselines and HoReN.

\paragraph{ROME.} We use the \citet{Meng2022LocatingAE} recipe verbatim: rewrite layer $\ell{=}5$, \texttt{subject\_last} fact-token, $25$ value-optimization steps with learning rate $0.5$, weight decay $10^{-3}$, KL factor $0.0625$, clamp norm factor $4$, and second-moment adjustment with $K_0$ statistics estimated from $10^5$ Wikipedia samples in float32. The value-loss layer is set to the last transformer block, i.e., $v_{\text{loss}} = n_{\text{layers}}-1$ for every LLM.

\paragraph{AlphaEdit.} We follow the \citet{Fang2024AlphaEditNC} LLaMA-3 configuration: a 5-layer rewrite window $\{4,5,6,7,8\}$, $\lambda{=}15{,}000$ for $K_0$ adjustment, $25$ value-optimization steps with learning rate $0.1$, weight decay $0.5$, clamp norm $0.75$, KL factor $0.0625$, $L_2$ regularization $10$, and null-space threshold $2{\times}10^{-2}$. The null-space projection $P$ is precomputed once per model from the same $10^5$-sample Wikipedia covariance. The value-loss layer is again $n_{\text{layers}}-1$ ($31$ for the 8B LLaMAs, $27$ for both 28-layer Qwen-based models, $63$ for the 64-layer SEA-LION-32B).

\paragraph{GRACE.} We use the \citet{hartvigsen2023aging} defaults: codebook learning rate $1.0$ (SGD) with $100$ inner steps, Euclidean distance, cold value initialization, \texttt{replace\_last} token replacement, coverage-based $\epsilon$-expansion with initial $\epsilon{=}1$ and early-stop regularization. Edits are applied to a single MLP \texttt{down\_proj} layer per LLM (see the cross-method layer alignment below).

\paragraph{WISE.} We adopt the \citet{Wang2024WISERT} Tab.~10 recipe: edit learning rate $1.0$ (SGD), $70$ inner iterations, mask ratio $\rho{=}0.2$, activation margins $(\alpha,\beta,\gamma){=}(5,20,10)$, activation ratio $0.88$, side-memory merge frequency $1000$ with TIES merging at density $0.53$ and weight $1.0$, and retrieval enabled. Norm constraint is $1.0$. Edits are applied to a single MLP \texttt{down\_proj} layer per LLM (see the cross-method layer alignment below).

\paragraph{REPAIR.} We follow the \citet{wang2025repair} configuration verbatim, using their official released hyperparameters (\citet{wang2025repair}, Appendix~C, Table~6) for every backbone we share with that paper: edit learning rate $0.9$ (SGD) with up to $10{,}000$ inner iterations, mask ratio $\rho{=}0.2$, error threshold $\tau_{\text{correct}}{=}0.85$, distillation weight $\lambda_{KD}{=}1.0$ with temperature $2.0$, routing margins $(\gamma_1,\gamma_2,\gamma)$ and activation ratio set to the per-model values listed in their Table~6 ($\gamma_1{=}2,\gamma_2{=}20,\gamma{=}10$, act.\ ratio $0.20$ for LLaMA-3 layer $29$; $\gamma_1{=}5,\gamma_2{=}20,\gamma{=}0$, act.\ ratio $0.88$ for Qwen2.5 layer $23$), and Weighted-TIES shard merging with loss-temperature $\lambda{=}0.2$. Edits are applied to the same MLP \texttt{down\_proj} layer used by GRACE/WISE/HoReN (see the cross-method layer alignment below); for backbones not studied by~\citet{wang2025repair} we transfer the same-family recipe (LLaMA-3 settings for the LLaMA-derived models and GPT-OSS, Qwen-2.5 settings for Qwen-derived models).

\paragraph{UltraEdit.} We use the \citet{Gu2025UltraEditTS} recipe: lifelong learning rate $10^{-6}$, \texttt{token=mask}, editor batch size $1024$, with the published \texttt{gate\_proj}~+~\texttt{up\_proj} layer ranges for LLaMA-3-8B (layers 11--15 / 18--24) and Qwen2.5-7B (layers 18--26). For DeepSeek-R1-Distill-Qwen-1.5B (28 layers) and Qwen-SEA-LION-v4-32B (64 layers), where no published recipe exists, we transfer the Qwen recipe (layers 16--20 and 18--26 respectively).

\subsubsection{HoReN Hyperparameters}
\label{app:horen-params}

HoReN uses $\beta{=}20$, $\gamma{=}0.1$, and $M{=}1$ across all experiments for the Hopfield-style refinement.

When a new entry is created (Algorithm~\ref{alg:hopfield_edit_simple}), the base model is frozen and only the codebook value adaptor parameters are optimized by Adam on the per-token cross-entropy loss of the target answer. Per-edit adaptor optimization uses learning rate $0.1$, up to $U{=}50$ optimization steps per edit, and early stopping when the per-token loss falls below $10^{-2}$ or fails to improve for 3 consecutive steps. No gradient flows back into the base model or into stored codebook keys; each codebook value adaptor is trained in isolation. The choice of layer where the HoReN codebook resides is specified in Appendix~\ref{app:cross-alignment-and-hardware}.

Two other model-specific parameters are required: the matching threshold $c$ and the token pooling ratio, both selected once per model via ablation; per-model values are listed in Table~\ref{tab:model-params}. HoReN uses tensor values as the default codebook value adaptor; the optional LoRA variant uses rank $r{=}4$.

Table~\ref{tab:model-params} lists the two model-specific parameters, matching threshold $c$ and token pooling ratio, for each backbone. Both are selected once per model via the ablation in Section~\ref{sec:expii-analysis-ablation} and held fixed across all datasets and edit scales.

\begin{table}[H]\small
\centering
\caption{Per-model hyperparameters used in all experiments.}
\label{tab:model-params}
\begin{tabular}{llcc}
\toprule
\textbf{Family} & \textbf{Model} & \textbf{Threshold $c$} & \textbf{Pooling ratio} \\
\midrule
\multirow{3}{*}{LLaMA-3/3.1}
  & LLaMA-3-8B-Instruct           & 0.80 & 50\% \\
  & LLaMA-3.1-8B-Instruct         & 0.85 & 60\% \\
  & DeepSeek-R1-Distill-Llama-8B  & 0.85 & 60\% \\
\midrule
\multirow{3}{*}{Qwen-2.5/3}
  & Qwen2.5-7B-Instruct           & 0.55 & 60\% \\
  & DeepSeek-R1-Distill-Qwen-1.5B & 0.60 & 60\% \\
  & Qwen-SEA-LION-v4-32B-IT       & 0.55 & 60\% \\
\midrule
GPT-OSS         & GPT-OSS-20B                   & 0.80 & 60\% \\
\bottomrule
\end{tabular}
\end{table}

Code, commands, and logs are available at \url{https://anonymous.4open.science/r/elephant-4B7F}. Portability guidelines are in Appendix~\ref{app:portability}.

\subsubsection{Cross-method alignment and Hardware}
\label{app:cross-alignment-and-hardware}

\paragraph{Cross-method layer alignment.} To avoid confounding HoReN's gains with a more favorable layer choice, GRACE, WISE, and HoReN are configured to edit the same layer per LLM: layer $29$ (LLaMA-3 / 3.1 / DeepSeek-R1-Distill-LLaMA-8B), layer $24$ (Qwen2.5-7B-Instruct), layer $27$ (DeepSeek-R1-Distill-Qwen-1.5B), layer $56$ (SEA-LION-32B), and layer $20$ of \texttt{self\_attn.o\_proj} (GPT-OSS-20B). These choices follow the published WISE/GRACE convention of editing a late MLP \texttt{down\_proj} block ($\approx n_{\text{layers}}-3$ to $n_{\text{layers}}-8$). ROME and AlphaEdit are excluded from this alignment because their original methodology prescribes different layer locations (causal-tracing layer $5$ for ROME; the rewrite window $\{4,\ldots,8\}$ for AlphaEdit) and changing them would deviate from those papers.

\paragraph{Hardware.} Experiments run on RTX~5090 in bf16, with GPT-OSS-20B run on RTX~Pro~6000 WS, and QWEN-SEA-LION-32B on H200.

\subsubsection{50K ZsRE Stress-Test Construction}
\label{app:50k-construction}

The 50K stress test in Figure~\ref{fig:50k-scaling} is drawn from a pool of 75{,}658 ZsRE-style samples aggregated from the standard ZsRE train split released by~\citet{Mitchell2021FastME} together with additional entries drawn from the original ZsRE source corpus~\citep{levy2017zero} that do not overlap with the standard split, yielding a pool large enough to support a stream of 50K non-overlapping edits. Each pool entry carries (i) an original prompt, (ii) a target answer, (iii) a paraphrased query, and (iv) a locality query, so the three evaluation sets ($\mathcal{D}_{\mathrm{edit}}, \mathcal{D}_{\mathrm{rephrase}}, \mathcal{D}_{\mathrm{locality}}$) are aligned one-to-one with each edit, matching the Section~\ref{sec:problem-formulation} formulation. For the stress test we sample 50K entries from this pool (with a fixed seed) and apply them as a single uninterrupted sequential stream on LLaMA-3.1-8B, chosen as a representative 8B-scale model; results on other backbones at smaller scales are consistent (Table~\ref{tab:scaling-10k}). Evaluation is carried out at checkpoints $N\in\{1\text{K},10\text{K},15\text{K},20\text{K},30\text{K},40\text{K},50\text{K}\}$ over the matching prefix of each dataset.

\subsubsection{Portability Guidelines}
\label{app:portability}

Model-specific parameters are limited to two choices: (1) the matching
threshold $c$ and (2) the token pooling ratio. Both are selected once
per model and then held fixed across datasets and edit scales. The
resulting values (Table~\ref{tab:model-params}) cluster tightly within
each family while shifting cleanly between families, indicating that
they reflect genuine embedding-geometry differences rather than
dataset- or benchmark-specific tuning.

Concretely, the calibrated thresholds fall into three distinct,
family-coherent bands:

\begin{itemize}
  \item \textbf{LLaMA-3 / 3.1 family} ($c \in [0.80, 0.85]$, span
    $0.05$): LLaMA-3-8B-Instruct uses $0.80$, while LLaMA-3.1-8B-Instruct
    and DeepSeek-R1-Distill-Llama-8B both use $0.85$. The fact that
    DeepSeek's distilled LLaMA inherits the LLaMA-3.1 threshold despite
    having undergone substantial reasoning-oriented post-training is
    direct evidence that $c$ tracks the underlying representation
    geometry of the base model rather than the downstream task.
  \item \textbf{Qwen-2.5 / 3 family} ($c \in [0.55, 0.60]$, span
    $0.05$): Qwen2.5-7B-Instruct and the 32B Qwen-SEA-LION-v4 both use
    $0.55$, and DeepSeek-R1-Distill-Qwen-1.5B uses $0.60$ -- again,
    the distilled variant remains within the family band. The
    consistent $\sim\!0.25$ gap from the LLaMA band reflects the
    smaller cosine spread of Qwen hidden states at the edited layer,
    not a per-experiment correction.
  \item \textbf{GPT-OSS} ($c = 0.80$): the only model in its family,
    sitting in the LLaMA range, consistent with its larger residual
    stream norms.
\end{itemize}

Pooling ratio is even more stable: $60\%$ for every model except
LLaMA-3-8B-Instruct ($50\%$), and chosen by a single ablation
(Figure~\ref{fig:token-selection}) rather than re-tuned per benchmark or
scale.

This level of calibration is consistent with existing methods: GRACE
requires per-model $\varepsilon$, WISE requires per-model importance
weights, and REPAIR requires per-model cross-layer settings -- in each
case the per-model knob also reflects representation-level
properties rather than dataset-specific tuning. For a new model we
recommend: (i) identify the family of the base model and initialize
$c$ to the family's central value ($0.85$ for LLaMA-derived models,
$0.55$ for Qwen-derived models); (ii) run the token-selection ablation
at moderate scale ($N{=}100$) to confirm the $60\%$ pooling default;
and (iii) if needed, refine $c$ with a short sweep over
	$\{0.55, 0.60, 0.80, 0.85\}$ on $30$ edits, anchored at the
	family-default value.

\section{Additional ZsRE Results}
\label{app:zsre-results}
\subsection{LLaMA-3.1-8B Large-Scale Stability}
\label{app:llama-scaling}

Table~\ref{tab:llama-scaling} reports full LLaMA-3.1-8B results at $N{=}2000$, $5000$, and $10000$ edits, including AlphaEdit and UltraEdit which are omitted from Table~\ref{tab:scaling-10k} due to prohibitive runtime on other models. AlphaEdit collapses between $N{=}2000$ and $N{=}5000$ (OP drops from $0.82$ to $0.10$) due to near-zero locality; UltraEdit degrades more gradually (OP from $0.73$ to $0.67$). HoReN maintains OP above $0.95$ throughout all scales.

\begin{table}[H]\scriptsize
\centering
\caption{LLaMA-3.1-8B results at large scale on ZsRE ($N{=}2000$, $5000$, $10000$). AlphaEdit and UltraEdit are included here to document their scaling behavior; they are excluded from Table~\ref{tab:scaling-10k} due to prohibitive runtime on other models.}
\label{tab:llama-scaling}
\begin{tabular}{l|cccc|cccc|cccc}
\toprule
\textbf{Method}
& \multicolumn{4}{c|}{\textbf{$N = 2000$}}
& \multicolumn{4}{c|}{\textbf{$N = 5000$}}
& \multicolumn{4}{c}{\textbf{$N = 10000$}} \\
\cmidrule(lr){2-5} \cmidrule(lr){6-9} \cmidrule(lr){10-13}
& \textbf{Rel.} & \textbf{Gen.} & \textbf{Loc.} & \textbf{OP}
& \textbf{Rel.} & \textbf{Gen.} & \textbf{Loc.} & \textbf{OP}
& \textbf{Rel.} & \textbf{Gen.} & \textbf{Loc.} & \textbf{OP} \\
\midrule
GRACE      & 0.99 & 0.05 & 1.00 & 0.37 & 0.99 & 0.05 & 1.00 & 0.37 & 0.99 & 0.06 & 1.00 & 0.38 \\
WISE       & 0.49 & 0.47 & 0.99 & 0.61 & 0.40 & 0.37 & 0.98 & 0.52 & 0.37 & 0.36 & 0.98 & 0.51\\
AlphaEdit  & 0.95 & 0.88 & 0.65 & 0.82 & 0.18 & 0.15 & 0.03 & 0.10 & 0.21 & 0.17 & 0.02 & 0.08 \\
UltraEdit  & 0.87 & 0.85 & 0.53 & 0.73 & 0.86 & 0.85 & 0.45 & 0.69 & 0.86 & 0.84 & 0.42 & 0.67 \\
\textit{HoReN} & \textbf{0.99} & \textbf{0.94} & \textbf{0.98} & \textbf{0.97} & \textbf{0.99} & \textbf{0.92} & \textbf{0.96} & \textbf{0.96} & \textbf{0.99} & \textbf{0.92} & \textbf{0.96} & \textbf{0.96} \\
\bottomrule
\end{tabular}
\end{table}

Figure~\ref{fig:stability-curve} tracks the full sequential trajectory up to 10K edits: HoReN's generalization stays flat while GRACE and WISE decay monotonically, and HoReN retains $99\%$ accuracy on early edits after the full 10K sequence with no catastrophic forgetting.

\begin{figure}[H]
\centering
\includegraphics[width=\textwidth]{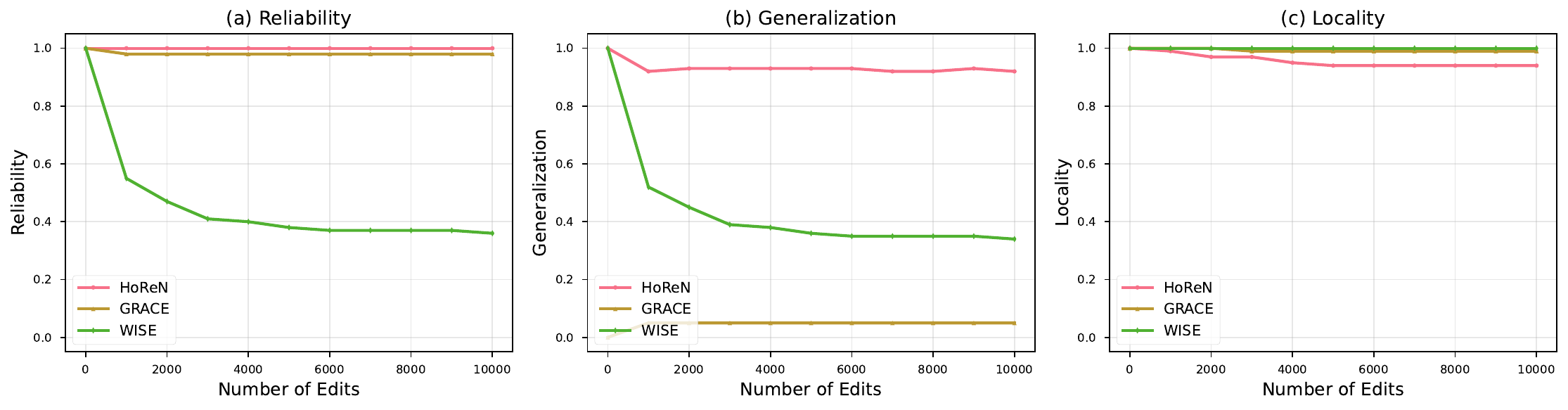}
\caption{Stability of editing performance over sequential edits on ZsRE with LLaMA-3.1-8B (up to 10K edits). HoReN maintains stable generalization while GRACE and WISE degrade.}
\label{fig:stability-curve}
\end{figure}

Figures~\ref{fig:50k-rel}--\ref{fig:50k-loc} decompose the OP curves from Figure~\ref{fig:50k-scaling} into the three constituent metrics for all methods on ZsRE (LLaMA-3.1-8B). HoReN maintains Reliability above 0.99 and Generalization above 0.91 throughout 50K edits, while Locality declines gradually from 1.00 to 0.89 the primary source of the slight OP drop. GRACE sustains near-perfect Reliability and Locality but has effectively zero Generalization at all scales, confirming the routing bottleneck rather than a weight-editing limitation. AlphaEdit's Locality collapse between 2K and 5K drives its OP cliff.

\begin{figure}[H]
    \centering
    \includegraphics[width=\textwidth]{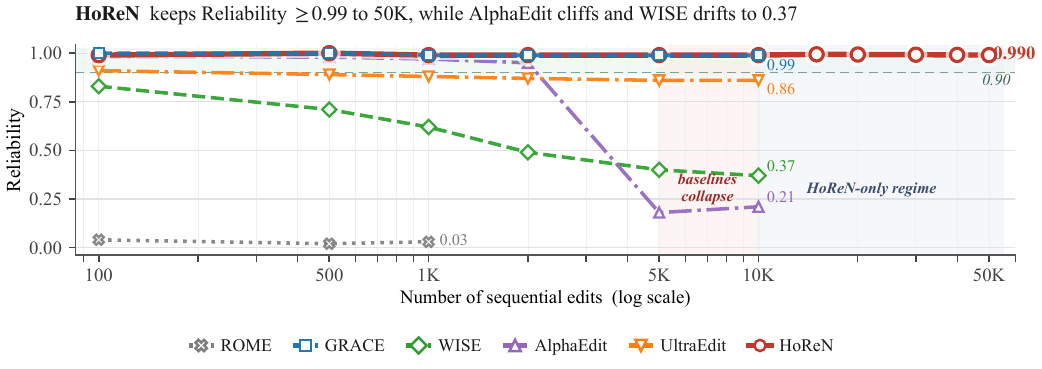}
    \caption{Reliability scaling to 50K sequential edits on ZsRE (LLaMA-3.1-8B).}
    \label{fig:50k-rel}
\end{figure}

\vspace{-0.4em}
\begin{figure}[H]
    \centering
    \includegraphics[width=\textwidth]{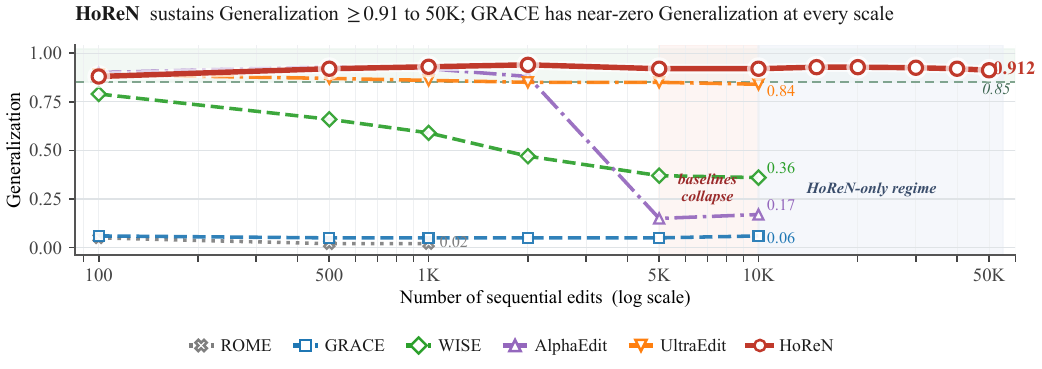}
    \caption{Generalization scaling to 50K sequential edits on ZsRE (LLaMA-3.1-8B).}
    \label{fig:50k-gen}
\end{figure}

\vspace{-0.4em}
\begin{figure}[H]
    \centering
    \includegraphics[width=\textwidth]{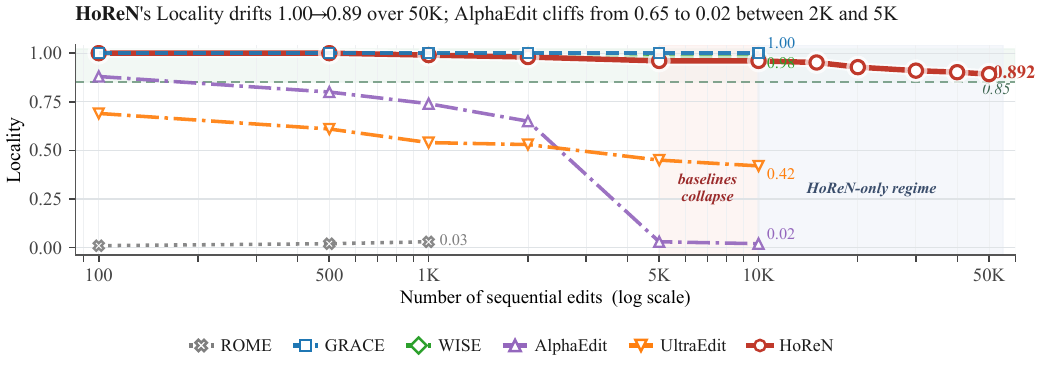}
    \caption{Locality scaling to 50K sequential edits on ZsRE (LLaMA-3.1-8B).}
    \label{fig:50k-loc}
\end{figure}

\subsection{Efficiency: Theory, Scaling, and Trade-offs}
\label{app:efficiency}

\begin{figure}[H]
    \centering
    \includegraphics[width=\textwidth]{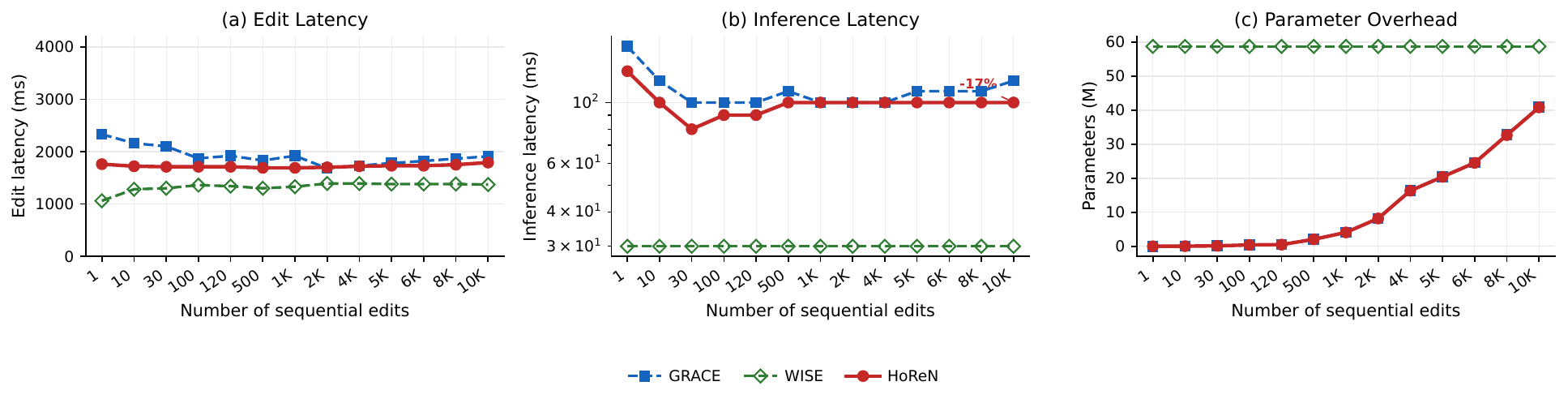}
    \caption{Efficiency scaling on ZsRE (LLaMA-3.1-8B) over $N\!\in\![1,10\text{K}]$ sequential edits: \textbf{(a)} per-edit latency, \textbf{(b)} per-query inference latency, and \textbf{(c)} cumulative trainable parameters of the side-memory overhead.}\label{fig:efficiency}
\end{figure}

\paragraph{Theoretical complexity.} When using direct tensor value as the codebook value adaptor, HoReN inherits GRACE's $O(Kd)$ retrieval cost over a codebook of size $K$ in dimension $d$~\citep{hartvigsen2023aging}, and adds a constant number of damped Hopfield refinement steps (softmax + matmul + normalization) per query. This refinement process is $O(Kd)$ per step and is performed only for a constant number of times, so HoReN's asymptotic cost is the same as GRACE's, which is confirmed by Figure~\ref{fig:efficiency}.

\paragraph{Editing efficiency.} On LLaMA-3.1-8B model, HoReN, GRACE, and WISE have similar editing time at around $1700$ ms, but on Qwen-SEA-LION-v4-32B (Table~\ref{tab:efficiency}), the gap widens substantially: HoReN edits in $4.85$s versus WISE at $17.37$s ($3.6\times$ faster) and AlphaEdit at $22.83$s ($4.7\times$ faster). Parameter-modifying baselines pay a steep activation-statistics / projection cost that scales with model width, while HoReN's per-edit cost is dominated by the small value adaptor and is largely insensitive to backbone size.

\paragraph{Inference-vs-memory trade-off.} WISE trades memory for latency in the opposite direction from HoReN: it pays a fixed $+58.72$M up front but routes inference through a small module ($30$ms). HoReN starts at $+4.09$M (1K edits) and grows by ${\sim}4.08$M per additional 1K edits, so its side-memory only catches up to WISE's at roughly $N\!\approx\!14$--$15$K edits, by which point WISE has already lost ${\geq}0.4$ on Generalization in our 50K-scaling study (Figure~\ref{fig:50k-scaling}). For the editing regimes we evaluate, HoReN therefore Pareto-dominates WISE on the (memory, generalization) frontier while remaining within ${\sim}5\%$ inference latency of GRACE.

\begin{table}[H]\scriptsize

\centering
\caption{Efficiency at $N{=}1000$ edits on ZsRE. ``Side mem.'' counts the editor's auxiliary parameters (codebook for HoReN/GRACE; dual-memory for WISE); ``$\Delta$ Model wts.'' counts changes to the backbone (AlphaEdit modifies the model in place, hence $0$ side-memory). Following the GRACE convention~\citep{hartvigsen2023aging}, side-memory excludes per-edit \emph{key} storage, which would roughly double the figure for HoReN/GRACE; AlphaEdit has no key storage, and WISE's reported value already includes its routing memory.}\label{tab:efficiency}
\begin{tabular}{llcccc}
\toprule
\textbf{Model} & \textbf{Method} & \textbf{Edit (ms)} & \textbf{Infer.\ (ms)} & \textbf{Side mem.\ (M)} & \textbf{$\Delta$ Model wts.\ (M)} \\
\midrule
\multirow{4}{*}{\begin{tabular}[c]{@{}l@{}}LLaMA-3.1\\8B\end{tabular}}
 & WISE      & 1670 & 30  & $+58.72$ & $0$\\
 & GRACE     & 1690 & 100 & $+4.10$  & $0$ \\
 & AlphaEdit & 3990 & 30  & $0$ & {$+293.60$} \\
 & \textit{HoReN} & \textbf{1670} & \textbf{100} & $\mathbf{+4.09}$ & {$\mathbf{0}$} \\
\midrule
\multirow{4}{*}{\begin{tabular}[c]{@{}l@{}}SEA-LION\\32B\end{tabular}}
 & WISE      & 17370 & 80  & $+131.07$ & {$0$} \\
 & GRACE     & 4800  & 330 & $+5.12$   & {$0$} \\
 & AlphaEdit & 22830 & 90  & {$0$} & {$+655.36$} \\
 & \textit{HoReN} & \textbf{4850} & \textbf{330} & $\mathbf{+5.09}$ & {$\mathbf{0}$} \\
\bottomrule
\end{tabular}
\end{table}

\section{Cross-Dataset Generalization}
\label{app:cross-dataset}

\subsection{Full Numerical Results}
\label{app:cross-dataset-tables}

Tables~\ref{tab:alphaedit}--\ref{tab:cross-dataset} report the complete numerical results underlying Figure~\ref{fig:cross-dataset} in Section~\ref{sec:expii-cross-dataset}.

\begin{table}[H]\scriptsize
\centering
\caption{Sequential editing on WikiBigEdit (LLaMA-3.1-8B). AlphaEdit's locality degrades sharply under long edit sequences while HoReN remains stable.}
\label{tab:alphaedit}
\begin{tabular}{clccc}
\toprule
\textbf{\#Edits} & \textbf{Method} & \textbf{Rel.} & \textbf{Gen.} & \textbf{Loc.} \\
\midrule
\multirow{2}{*}{500}
 & AlphaEdit & 0.825 & 0.605 & 0.267 \\
 & \textbf{HoReN} & \textbf{0.992} & \textbf{0.767} & \textbf{0.682} \\
\midrule
\multirow{2}{*}{1000}
 & AlphaEdit & 0.797 & 0.618 & 0.194 \\
 & \textbf{HoReN} & \textbf{0.987} & \textbf{0.753} & \textbf{0.654} \\
\midrule
\multirow{2}{*}{3000}
 & AlphaEdit & 0.799 & 0.651 & 0.169 \\
 & \textbf{HoReN} & \textbf{0.986} & \textbf{0.754} & \textbf{0.631} \\
\bottomrule
\end{tabular}
\end{table}

\begin{table}[H]\scriptsize
\centering
\caption{Unstructured knowledge editing on UnKE (LLaMA-3.1-8B) broken down by question type. Rows: three UnKEBench question types (\emph{Original}, \emph{Paraphrase}, \emph{Sub-question}) crossed with all applicable metrics. Columns: edit counts $N\!\in\!\{10,100,500\}$, each split between UnKE and HoReN. Best per (row, $N$) in bold; BLEU and BERTScore are not reported for the Sub-question split in the original UnKEBench protocol~\citep{Deng2024UnKEUK}.}
\label{tab:unke-breakdown}
\setlength{\tabcolsep}{4pt}
\begin{tabular}{ll cc cc cc}
\toprule
& & \multicolumn{2}{c}{$N{=}10$} & \multicolumn{2}{c}{$N{=}100$} & \multicolumn{2}{c}{$N{=}500$} \\
\cmidrule(lr){3-4} \cmidrule(lr){5-6} \cmidrule(lr){7-8}
\textbf{Q-type} & \textbf{Metric} & UnKE & HoReN & UnKE & HoReN & UnKE & HoReN \\
\midrule
\multirow{5}{*}{Original}
 & BLEU       & \textbf{0.191} & 0.143          & \textbf{0.202} & 0.126          & 0.098 & \textbf{0.117} \\
 & ROUGE-1    & 0.453          & \textbf{0.583} & 0.432          & \textbf{0.457} & 0.222 & \textbf{0.404} \\
 & ROUGE-2    & 0.240          & \textbf{0.395} & 0.196          & \textbf{0.219} & 0.054 & \textbf{0.184} \\
 & ROUGE-L    & 0.431          & \textbf{0.551} & 0.405          & \textbf{0.428} & 0.205 & \textbf{0.380} \\
 & BERTScore  & 0.713          & \textbf{0.814} & 0.698          & \textbf{0.716} & 0.183 & \textbf{0.682} \\
\midrule
\multirow{5}{*}{Paraphrase}
 & BLEU       & \textbf{0.204} & 0.126          & \textbf{0.199} & 0.115          & 0.097 & \textbf{0.106} \\
 & ROUGE-1    & 0.452          & \textbf{0.476} & \textbf{0.418} & 0.402          & 0.215 & \textbf{0.360} \\
 & ROUGE-2    & 0.211          & \textbf{0.278} & \textbf{0.178} & 0.165          & 0.047 & \textbf{0.142} \\
 & ROUGE-L    & 0.443          & \textbf{0.452} & \textbf{0.392} & 0.374          & 0.197 & \textbf{0.338} \\
 & BERTScore  & \textbf{0.756} & 0.707          & 0.676          & \textbf{0.685} & 0.166 & \textbf{0.659} \\
\midrule
\multirow{5}{*}{Sub-question}
 & ROUGE-1    & \textbf{0.469} & 0.424          & 0.309          & \textbf{0.373} & 0.141 & \textbf{0.338} \\
 & ROUGE-2    & \textbf{0.176} & 0.168          & 0.071          & \textbf{0.123} & 0.016 & \textbf{0.100} \\
 & ROUGE-L    & \textbf{0.459} & 0.417          & 0.303          & \textbf{0.364} & 0.137 & \textbf{0.332} \\
\bottomrule
\end{tabular}
\end{table}

\noindent\textbf{Column / row semantics (UnKEBench~\citep{Deng2024UnKEUK}).}
For each unstructured edit text $A$, UnKEBench provides three question types probing complementary aspects of the edited model's behavior. The \textbf{Original} block measures \emph{reliability}: whether the edit is faithfully recalled on the original question $Q$. The \textbf{Paraphrase} block measures \emph{generalization}: whether the edit transfers to a semantically equivalent rewording $Q_p$ of $Q$. The \textbf{Sub-question} block measures \emph{factual coverage}: whether individual knowledge entities within the passage are correctly installed. \textbf{BLEU} and \textbf{ROUGE-1/2/L} are word-level $n$-gram overlap scores; \textbf{BERTScore} captures semantic similarity beyond lexical match. BLEU and BERTScore are not reported for the Sub-question split in the original UnKEBench protocol.

\begin{table}[H]\scriptsize
\centering
\caption{Cross-dataset generalization on WikiBigEdit and ZsRE at $N{=}1000$ edits (DeepSeek-R1-Distill-Qwen-1.5B). Gray superscripts show HoReN's absolute improvements over each baseline.}
\label{tab:cross-dataset}
\begin{tabular}{lcc|cc}
\toprule
\multirow{2}{*}{\textbf{Method}}
& \multicolumn{2}{c|}{\textbf{WikiBigEdit}}
& \multicolumn{2}{c}{\textbf{ZsRE}} \\
\cmidrule(lr){2-3} \cmidrule(lr){4-5}
& \textbf{Rel.} & \textbf{Gen.} & \textbf{Rel.} & \textbf{Gen.} \\
\midrule
GRACE  & 0.63\textsuperscript{\textcolor{gray}{\textbf{\scriptsize +0.33}}} & 0.07\textsuperscript{\textcolor{gray}{\textbf{\scriptsize +0.72}}} & 0.83\textsuperscript{\textcolor{gray}{\textbf{\scriptsize +0.11}}} & 0.05\textsuperscript{\textcolor{gray}{\textbf{\scriptsize +0.74}}} \\
WISE   & 0.47\textsuperscript{\textcolor{gray}{\textbf{\scriptsize +0.49}}} & 0.38\textsuperscript{\textcolor{gray}{\textbf{\scriptsize +0.41}}} & 0.49\textsuperscript{\textcolor{gray}{\textbf{\scriptsize +0.45}}} & 0.47\textsuperscript{\textcolor{gray}{\textbf{\scriptsize +0.32}}} \\
REPAIR & 0.58\textsuperscript{\textcolor{gray}{\textbf{\scriptsize +0.38}}} & 0.54\textsuperscript{\textcolor{gray}{\textbf{\scriptsize +0.25}}} & 0.59\textsuperscript{\textcolor{gray}{\textbf{\scriptsize +0.35}}} & 0.57\textsuperscript{\textcolor{gray}{\textbf{\scriptsize +0.22}}} \\
\midrule
\textbf{HoReN}  & \textbf{0.96} & \textbf{0.79} & \textbf{0.94} & \textbf{0.79} \\
\bottomrule
\end{tabular}
\end{table}

\section{Analysis and Ablation Details}
\label{app:analysis-ablation}

\subsection{Routing Ingredients: GRACE $\to$ GRACE+Norm. $\to$ HoReN}
\label{app:grace-vs-horen}

Table~\ref{tab:grace-vs-horen} isolates the contribution of each routing ingredient. We fix the prompt-token pooling strategy across methods: queries are obtained by averaging the hidden states of the last 60\% of prompt tokens for both GRACE and HoReN. This controls for pooling choice, which would otherwise act as a confounding factor.  Table~\ref{tab:normalized-codebook} confirms that applying Hopfield updates \emph{without} normalization collapses locality, motivating the joint design.

\begin{table}[H]\scriptsize
\centering
\caption{Routing ingredient ablation: unnormalized GRACE vs.\ GRACE+Norm. vs.\ HoReN (ZsRE, LLaMA-3.1-8B).}
\label{tab:grace-vs-horen}
\begin{tabular}{lccccc}
\toprule
\textbf{N} & \textbf{Method} & \textbf{Rel.} & \textbf{Gen.} & \textbf{Loc.} & \textbf{OP} \\
\midrule
\multirow{3}{*}{$N=100$}
 & \textbf{HoReN} & \textbf{0.99} & \textbf{0.88} & \textbf{1.00} & \textbf{0.96} \\
 & GRACE+Norm. & 1.0 & 0.41 & 1.0 & 0.75 \\
 & GRACE & 1.0 & 0.06 & 1.0 & 0.40 \\
\midrule
\multirow{3}{*}{$N=500$}
 & \textbf{HoReN} & \textbf{1.00} & \textbf{0.92} & \textbf{1.00} & \textbf{0.97} \\
 & GRACE+Norm. & 1.0 & 0.37 & 1.0 & 0.72 \\
 & GRACE & 1.0 & 0.06 & 1.0 & 0.39 \\
\midrule
\multirow{3}{*}{$N=1000$}
 & \textbf{HoReN} & \textbf{0.99} & \textbf{0.93} & \textbf{0.99} & \textbf{0.97} \\
 & GRACE+Norm. & 1.0 & 0.33 & 1.0 & 0.69 \\
 & GRACE & 1.0 & 0.07 & 1.0 & 0.41 \\
\bottomrule
\end{tabular}
\end{table}

\begin{table}[H]\scriptsize
\centering
\caption{Effect of normalization on Hopfield retrieval (ZsRE, LLaMA-3.1-8B). Without normalization, the Hopfield step collapses locality.}
\label{tab:normalized-codebook}
\begin{tabular}{lccccc}
\toprule
\textbf{N} & \textbf{Method} & \textbf{Rel.} & \textbf{Gen.} & \textbf{Loc.} & \textbf{OP} \\
\midrule
\multirow{2}{*}{$N=100$}
 & \textbf{Normalized} & \textbf{0.99} & \textbf{0.88} & \textbf{1.00} & \textbf{0.96} \\
 & Unnormalized & 1.00 & 0.92 & 0.08 & 0.43 \\
\midrule
\multirow{2}{*}{$N=500$}
 & \textbf{Normalized} & \textbf{1.00} & \textbf{0.92} & \textbf{1.00} & \textbf{0.97} \\
 & Unnormalized & 0.98 & 0.79 & 0.07 & 0.38 \\
\midrule
\multirow{2}{*}{$N=1000$}
 & \textbf{Normalized} & \textbf{0.99} & \textbf{0.93} & \textbf{0.99} & \textbf{0.97} \\
 & Unnormalized & 0.96 & 0.75 & 0.07 & 0.37 \\
\bottomrule
\end{tabular}
\end{table}

\subsection{Hopfield Iteration Steps}
\label{app:hopfield-iter}

\begin{figure}[H]
\centering
\includegraphics[width=\textwidth]{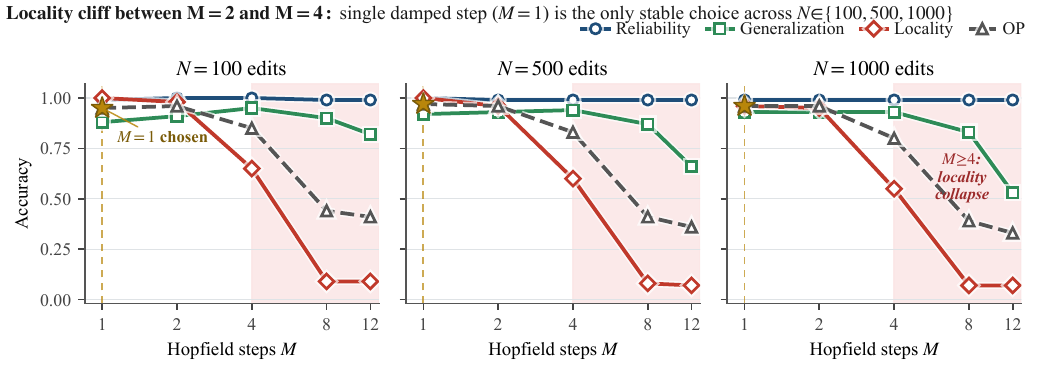}
\caption{Effect of Hopfield refinement steps $M$ on the four editing metrics across edit counts $N\!\in\!\{100,500,1000\}$ (LLaMA-3.1-8B, ZsRE). Reliability is essentially flat in $M$; Generalization peaks at $M\!\in\![1,4]$; Locality cliffs from ${\sim}0.95$ to ${\sim}0.6$ between $M{=}2$ and $M{=}4$ and bottoms out near $0.07$ for $M\!\geq\!8$, dragging OP with it. The single damped step ($M{=}1$, gold star) is the only value of $M$ that keeps all three metrics jointly stable across all three scales.}
\label{fig:hopfield-steps}
\end{figure}


\subsection{Value Adaptor: Direct Vector vs.\ LoRA}
\label{app:lora-vs-value}

Direct-vector and LoRA value adaptors achieve identical editing performance; the direct variant is preferred for its $8{\times}$ lower parameter cost ($+4.09$M parameters versus $+32.64$M at $N{=}1000$ ). Table~\ref{tab:lora-vs-value-app} compares the two across $N{\in}\{100,500,1000\}$.

Both variants reach OP $0.95$--$0.98$ at every scale, with differences of at most $0.01$. The equivalence is expected: HoReN's retrieval accuracy depends entirely on the key-matching mechanism (normalization + Hopfield step), not on the capacity of the value payload. Once the correct edit is retrieved, either adaptor has sufficient expressivity to represent the target answer at these scales. The direct variant adds $+4.09$M parameters at $N{=}1000$ versus $+32.64$M for LoRA, an $8{\times}$ gap with no performance benefit. We therefore recommend direct-vector as the default; LoRA remains available for deployments requiring strict parameter isolation between edits (e.g., federated or multi-tenant settings).

\begin{table}[H]\scriptsize
\centering
\caption{Direct value vs.\ LoRA ($r{=}4$) on ZsRE, LLaMA-3.1-8B. Performance is identical; direct variant adds $8\times$ fewer parameters.}
\label{tab:lora-vs-value-app}
\begin{tabular}{lccccc}
\toprule
\textbf{N} & \textbf{Method} & \textbf{Rel.} & \textbf{Gen.} & \textbf{Loc.} & \textbf{OP} \\
\midrule
\multirow{2}{*}{$N=100$}  & Value & 0.99 & 0.88 & 1.00 & 0.96 \\
                          & LoRA  & 0.99 & 0.87 & 1.00 & 0.95 \\
\midrule
\multirow{2}{*}{$N=500$}  & Value & 1.00 & 0.92 & 1.00 & 0.97 \\
                          & LoRA  & 1.00 & 0.92 & 1.00 & 0.97 \\
\midrule
\multirow{2}{*}{$N=1000$} & Value & 0.99 & 0.93 & 0.99 & 0.97 \\
                          & LoRA  & 1.00 & 0.93 & 1.00 & 0.98 \\
\bottomrule
\end{tabular}
\end{table}

\subsection{Key/Query Representation: Token Pooling Ratio}
\label{app:token-selection}

\textbf{The last 60\% suffix pooling ratio is Pareto-optimal across all scales; performance is non-monotonic, with both short and long suffixes introducing distinct failure modes.} Figure~\ref{fig:token-selection} sweeps ten ratios from 10\% to 100\% on LLaMA-3.1-8B, ZsRE, at $N{\in}\{100,500,1000\}$.

\emph{Short suffixes ($\leq$30\%) — generalization failure.} Very short suffixes pool too few tokens to capture the semantic content of the prompt, producing a query representation that reflects surface position rather than meaning. Gen.\ stays at $0.47$--$0.67$ at $N{=}1000$ regardless of $N$: the key constructed from the original prompt's last few tokens is insufficiently similar to the key from the paraphrase's last few tokens, so the Hopfield step cannot bridge the residual angular gap.

\emph{Long suffixes ($\geq$80\%) — locality failure.} Very long suffixes include tokens specific to the prompt's preamble or instruction format, making the query representation sensitive to surface-level prompt differences rather than semantic content. This causes locality collapse: at $N{=}1000$, Loc.\ drops to $0.74$--$0.84$ as unrelated locality queries, which share similar preambles with edit queries, begin triggering false matches.

\emph{Last 60\% — Pareto-optimal.} The 60\% suffix captures enough semantic context for robust paraphrase matching (Gen.\ $0.88$--$0.93$) while excluding enough of the prompt-specific prefix to keep locality queries well-separated (Loc.\ $0.99$--$1.00$). The $50$--$70\%$ range all reach OP ${\geq}0.95$, indicating the choice is not sensitive to exact calibration. We treat the 60\% default as an empirical heuristic; learned token selection (e.g., attention-weighted pooling) is a concrete future direction.

\begin{figure}[H]
\centering
\includegraphics[width=\textwidth]{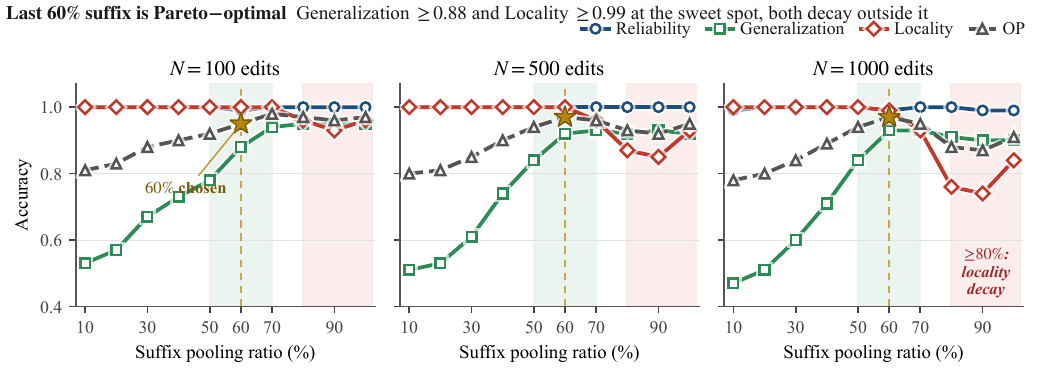}
\caption{Effect of suffix pooling ratio on the four editing metrics across $N\!\in\!\{100,500,1000\}$ (LLaMA-3.1-8B, ZsRE). Reliability is essentially flat at $1.00$; Generalization grows monotonically from ${\sim}0.50$ at $10\%$ to ${\sim}0.93$ from $60\%$ onward; Locality is held at $1.00$ up to $60\%$ then decays as longer suffixes start incorporating prompt-specific tokens, dropping to $0.74\!-\!0.84$ at $\geq 80\%$ for $N{=}1000$. The $50\%\!-\!70\%$ Pareto band (green) all reach OP $\geq 0.94$; the $\geq 80\%$ locality-decay zone (red) cuts OP. The chosen $60\%$ setting (gold star) is the only ratio that simultaneously achieves Generalization $\geq 0.88$ and Locality $\geq 0.99$ at every scale.}
\label{fig:token-selection}
\end{figure}

\end{document}